\documentclass[11pt]{article} 
\usepackage{fullpage}
\usepackage[utf8]{inputenc} %
\usepackage[T1]{fontenc}    %
\usepackage{hyperref}       %
\usepackage{url}            %
\usepackage{booktabs}       %
\usepackage{amsfonts}       %
\usepackage{nicefrac}       %
\usepackage{microtype}      %
\usepackage{xcolor}         %
\usepackage[affil-sl]{authblk}
\usepackage{natbib}

\usepackage{amsmath, amssymb, amsthm}

\title{\textbf{Intrinsic dimensionality and generalization properties of the $\mathcal{R}$-norm inductive bias}}
\author{Navid Ardeshir}
\author{Daniel Hsu}
\author{Clayton Sanford}
\affil{Columbia University}

\usepackage{xcolor}
\usepackage{bbm}
\usepackage{xspace}
\usepackage{tikz}
\usepackage{enumitem}
\setlist{leftmargin=0.5cm}
\usepackage{mathtools}
\usepackage{multirow}

\usepackage{booktabs}

\newtheorem{theorem}{Theorem}
\newtheorem{remark}{Remark}
\newtheorem{proposition}{Proposition}

\newtheorem{lemma}[theorem]{Lemma}
\newtheorem{fact}{Fact}
\newtheorem{itheorem}{Informal Theorem}

\newcommand\ProbabilitySymbol{\mathbb{P}}
\newcommand\ExpectationSymbol{\mathbb{E}}

\newcommand{\clip}{\operatorname{clip}}

\DeclareMathOperator*{\argmin}{arg\,min}

\DeclareMathOperator*{\poly}{poly}
\DeclareMathOperator*{\polylog}{polylog}

\newcommand{\expectationl}[2][]{\operatorname{\ExpectationSymbol}\displaylimits_{#1}[ #2 ]}
\newcommand{\EE}[1]{\operatorname{\ExpectationSymbol}\left[ #1 \right]}
\newcommand{\EEl}[1]{\operatorname{\ExpectationSymbol}[ #1 ]}
\newcommand{\indicator}[1]{\mathbbm{1}\left\{ #1 \right\}}
\newcommand{\indicatorl}[1]{\mathbbm{1}\setl{ #1 }}
\newcommand{\Var}[2][]{\mathop{\operatorname{Var}}_{#1}\left( #2\right)}
\newcommand{\Varl}[2][]{\mathop{\operatorname{Var}}_{#1}(#2)}

\newcommand{\Prob}{\operatorname{\ProbabilitySymbol}}

\newcommand{\ellone}[1][]{{L_{1} \ifthenelse{\isempty{#1}}{}{ \left( #1 \right) }
}}
\newcommand{\pr}[1]{\Prob\left[ #1 \right]}
\newcommand{\prl}[1]{\Prob[ #1 ]}
\newcommand{\E}[1]{\operatorname{\ExpectationSymbol}\left[ #1 \right]}

\newcommand{\norm}[2][]{\left\| #2 \right\|_{#1}}
\newcommand{\norml}[2][]{\| #2 \|_{#1}}

\newcommand{\innerprod}[2]{\left\langle #1, #2 \right\rangle}
\newcommand{\innerprodl}[2]{\langle #1, #2 \rangle}

\newcommand{\sign}[1]{\operatorname{sign}\left( #1 \right)}
\newcommand{\signl}[1]{\operatorname{sign}( #1 )}
\newcommand{\relu}{\varphi_{\operatorname{r}}}
\newcommand{\parity}{\chi}

\newcommand{\floor}[1]{\left\lfloor #1 \right\rfloor}
\newcommand{\floorl}[1]{\lfloor #1 \rfloor}
\newcommand{\ceil}[1]{\left\lceil #1 \right\rceil}
\newcommand{\ceill}[1]{\lceil #1 \rceil}
\newcommand{\abs}[1]{\left| #1 \right|}
\newcommand{\absl}[1]{| #1 |}
\newcommand{\set}[1]{\left\{ #1 \right\}}
\newcommand{\setl}[1]{\{ #1 \}}
\newcommand{\paren}[1]{\left( #1 \right)}
\newcommand{\parenl}[1]{( #1 )}
\newcommand{\bracket}[1]{\left[#1\right]}

\newcommand{\R}{\mathbb{R}}
\newcommand{\Z}{\mathbb{Z}}

\newcommand{\calR}{\mathcal{R}}

\newcommand{\flip}{\{\pm 1\}}
\newcommand{\flipd}{\{\pm \frac1{\sqrt{d}}\}}

\newcommand\bA{{\mathbf{A}}}

\newcommand\bB{{\mathbf{B}}}

\newcommand\bg{{\mathbf{g}}}
\newcommand{\bmu}{\boldsymbol{\mu}}
\newcommand{\bc}{\mathbf{c}}
\newcommand\bG{{\mathbf{G}}}

\newcommand\bh{{\mathbf{h}}}

\newcommand\bu{{\mathbf{u}}}

\newcommand\bv{{\mathbf{v}}}
\newcommand\br{{\mathbf{r}}}

\newcommand\bw{{\mathbf{w}}}

\newcommand\bx{{\mathbf{x}}}
\newcommand\bX{{\mathbf{X}}}
\newcommand\by{{\mathbf{y}}}

\newcommand\bZ{{\mathbf{Z}}}

\newcommand\beps{{\boldsymbol{\epsilon}}}

\newcommand{\fracl}[2]{#1 / #2}

\newcommand\intco[1]{[#1)}

\newcommand\T{{\scriptscriptstyle{\mathsf{T}}}}

\usepackage{thmtools,thm-restate}
\usepackage{mathrsfs}

\newcommand{\rnorm}[1]{\norm[\mathcal{R}]{#1}}
\newcommand{\rnorml}[1]{\norml[\mathcal{R}]{#1}}
\newcommand{\vnorm}[1]{\norm[\mathscr{V}_2]{#1}}
\newcommand{\vnorml}[1]{\norml[\mathscr{V}_2]{#1}}
\newcommand{\tv}[1]{\norm[\operatorname{TV}]{#1}}

\newcommand{\Rad}{\operatorname{Rad}}

\newcommand{\sph}{\mathbb{S}^{d-1}}

\newcommand{\domain}{\varOmega}
\newcommand{\support}{\operatorname{supp}}
\newcommand{\nun}{\boldsymbol{\nu}_n}
\newcommand\simiid{\sim_{\operatorname{iid}}}
\usepackage[normalem]{ulem}

\newcommand{\Lp}[2]{L^{#1}{(#2)}}
\newcommand{\Lpu}[1]{\Lp{#1}{\nu}}
\newcommand{\Lpemp}[1]{\Lp{#1}{\nun}}
\newcommand{\cubenormp}[2]{\norm[\Lpu{#2}]{#1}}
\newcommand{\cubenormpl}[2]{\norml[\Lpu{#2}]{#1}}
\newcommand{\cubenormtwo}[1]{\cubenormp{#1}{2}}
\newcommand{\cubenormtwol}[1]{\cubenormpl{#1}{2}}
\newcommand{\cubenorminfty}[1]{\cubenormp{#1}{\infty}}

\newcommand{\cubeinnerprod}[1]{\left\langle #1 \right\rangle_{\Lpu{2}}}
\newcommand{\cubeinnerprodl}[1]{\langle #1 \rangle_{\Lpu{2}}}
\newcommand{\empnorm}[1]{ \norm[ \Lpemp{2} ]{#1} }
\newcommand{\empinnerprod}[1]{\left\langle #1 \right\rangle_{\Lpemp{2}}}

\newcommand{\unif}{\operatorname{Unif}}

\newcommand{\Rnorm}{$\mathcal{R}$-norm\xspace}

\newcommand{\measures}{\mathcal{M}}
\newcommand\vmeasures{\mathcal{M}'}

\DeclareMathOperator\dd{d\!}

\newcommand{\tr}{\operatorname{tr}}
\newcommand\ridge{\operatorname{Ridge}}

\newcommand{\normnuzero}[1]{\norm[L^2(\nu_0)]{#1}}
\newcommand{\normnuzeroinf}[1]{\norm[L^\infty(\nu_0)]{#1}}
\newcommand{\innerprodnuzero}[2]{\innerprod{#1}{#2}_{L^2(\nu_0)}}

\makeatletter
\newcount\my@repeat@count
\newcommand{\myrepeat}[2]{\begingroup
  \my@repeat@count=\z@
  \@whilenum\my@repeat@count<#1\do{#2\advance\my@repeat@count\@ne}\endgroup
}
\makeatother

\DeclareMathOperator*{\esssup}{ess\,sup}
\DeclareMathOperator{\Leb}{Leb}

\begin{document}
{\def\thefootnote{}\footnotetext{E-mail: \texttt{na2844@columbia.edu}, \texttt{djhsu@cs.columbia.edu}, \texttt{clayton@cs.columbia.edu}}}

\maketitle

\begin{abstract}We study the structural and statistical properties of $\mathcal{R}$-norm minimizing interpolants of datasets labeled by specific target functions. The $\mathcal{R}$-norm is the basis of an inductive bias for two-layer neural networks, recently introduced to capture the functional effect of controlling the size of network weights, independently of the network width. We find that these interpolants are intrinsically multivariate functions, even when there are ridge functions that fit the data, and also that the $\mathcal{R}$-norm inductive bias is not sufficient for achieving statistically optimal generalization for certain learning problems. Altogether, these results shed new light on an inductive bias that is connected to practical neural network training.
 \end{abstract}

\section{Introduction}\label{sec:intro}

The study of inductive biases in neural network learning is important for theoretical understanding and for developing practical guidance in network training.
Recent theories of generalization rely on inductive biases of training algorithms to explain how neural nets that (nearly) interpolate training data can be accurate out-of-sample~\citep{neyshabur2015search,zhang2021understanding}.
When inductive biases are made explicit and their effects are elucidated, they can be incorporated into training procedures when deemed appropriate for a problem at hand.

In this paper, we study the inductive bias for two-layer neural networks implied by a variational norm called the \emph{\Rnorm}, introduced by~\citet{sess19} and \citet{owss19} to capture the functional effect of controlling the size of network weights.
(A definition is given in Section~\ref{ssec:rnorm}.)
We focus on the \emph{approximation} and \emph{generalization} consequences of preferring networks with small \Rnorm in the context of learning explicit target functions.
It is well-known that the size of the weights can play a critical role in generalization properties of neural networks~\citep{bartlett1996valid}, and weight-decay regularization is a common practice in gradient-based training~\citep{hinton1987learning,hanson1988comparing}.
Thus, explicating the consequences of the \Rnorm inductive bias may advance our understanding of generalization in practical settings.

We investigate the $d$-dimensional variational problem \eqref{interpolating-r-norm}, which seeks a neural net $g \colon \domain \to \R$ of minimum \Rnorm among those that perfectly fit a given labeled dataset $\setl{(x_i, y_i)}_{i \in [n]} \subset \domain \times \R$:
\begin{alignat}{2}
    \label{interpolating-r-norm}
    \tag{VP}
    \inf_{g \colon \domain \to \R} \rnorm{g}
    & \quad \text{s.t.} \quad \;\; g(x_i) = y_i
    & \quad & \forall i \in [n] ;
    \\
    \intertext{as well as a variant \eqref{approximate-r-norm} that only requires $g$ to uniformly approximate labels up to error $\epsilon \in (0,1)$:}
    \label{approximate-r-norm}
    \tag{$\epsilon$-VP}
    \inf_{g \colon \domain \to \R} \rnorm{g}
    & \quad \text{s.t.} \quad \abs{g(x_i) - y_i} \leq \epsilon
    & & \forall i \in [n] .
\end{alignat}
Here, $\domain \subset \R^d$ is a $d$-dimensional domain of interest.
We study structural and statistical properties of solutions to these problems for datasets labeled by specific target functions in high dimensions.

The recent introduction of the \Rnorm
and its connections to
weight-decay regularization
have catalyzed research on the foundational properties of solutions to \eqref{interpolating-r-norm}.
In particular, solutions in the one-dimensional ($d=1$) setting have been precisely characterized and their generalization properties are now well-understood by their connections to splines~\citep{dduf21, sess19, pn21a, hanin21}.
However, far less is known about the solutions of \Rnorm-minimizing interpolants for the general $d$-dimensional case.

\paragraph{Key message.}
Inductive biases based on certain variational norms, such as the \Rnorm, are believed to offer a way around the curse of dimensionality suffered by kernel methods, because they are adaptive to low-dimensional structure.
Researchers have pointed to this adaptivity property in non-parametric settings~\citep{b17,pn21b} and specific learning tasks with low-dimensional structure~\citep{wei2019regularization} as mathematical evidence of the statistical advantage of neural networks over kernel methods.
One may hypothesize that
the \Rnorm inductive bias achieves this advantage by favoring functions with low-dimensional structure.
Indeed, many other forms of inductive bias used in statistics and machine learning are known to explicitly identify relevant low-dimensional structure~\citep{candes2006robust,donoho2006compressed,candes2009exact,bhojanapalli2016global,begkmz22,damian2022neural,frei2022random,mousavi2022neural,galanti2022sgd}. Our results provide theoretical evidence that this is not always the case with the \Rnorm inductive bias, in a very strong sense that becomes more pronounced in higher dimensions.

We show that, even in cases where the dataset can be perfectly fit by an intrinsically one-dimensional function, the solutions $g$ to \eqref{interpolating-r-norm} or \eqref{approximate-r-norm} are not necessarily the piecewise-linear ridge functions described in previous works~\citep{sess19, hanin21}.
Rather, the \Rnorm is far better minimized by a \emph{multi-directional}\footnote{By a multi-directional function, we mean a function that does not \emph{only} depend on a one-dimensional projection of its input---i.e., a function that is not a ridge function (defined in Section~\ref{ssec:notation}).}
neural network $g$ that averages several ridge functions pointing in different directions, each of which approximates a small fraction of the data.

\subsection{Our contributions}\label{ssec:contrib}

Our results are summarized by the following informal theorems concerning the structural and generalization properties of $\mathcal{R}$-norm interpolation.
Together, they show that
the $\mathcal{R}$-norm inductive bias (1) leads to interpolants that are qualitatively different from those that minimize width or intrinsic dimensionality of the learned network, and (2) is insufficient for obtaining optimal generalization for a well-studied learning problem.

\begin{itheorem}[\Rnorm minimizers of the parity dataset are not ridge functions]\label{ithm:approx}
  \sloppy
  Suppose the dataset $\{ (x_i,y_i) \}_{i\in[n]} \subset \flip^d \times \flip$ used in \eqref{interpolating-r-norm} and \eqref{approximate-r-norm} is the complete dataset of $2^d$ examples labeled by the $d$-variable parity function.
  \begin{itemize}
    \item The optimal value of \eqref{interpolating-r-norm} is $\Theta(d)$.
    \item The optimal value of \eqref{approximate-r-norm} for any $\epsilon\in\intco{0,1/2}$---with the additional constraint that $g$ be a ridge function---is $\Theta(d^{3/2})$.
  \end{itemize}
  \fussy
\end{itheorem}
This result is presented formally in Section~\ref{sec:parity-approx}.
In Section~\ref{ssec:parity-rnorm-ridge-approx-lb},
we show that every ridge function satisfying the constraints of \eqref{approximate-r-norm} has \Rnorm at least $\Omega(d^{3/2})$; this bound is tight for ridge functions, as there is a matching upper bound.
Using an averaging strategy, we show in Section~\ref{ssec:parity-rnorm-ub} the existence of multi-directional interpolants $g$ of the parity dataset with $\rnorm{g} = O(d)$, and we also establish the optimality of this construction in Section~\ref{ssec:parity-rnorm-approx-lb}.
These results characterize the optimal value of \eqref{interpolating-r-norm} in terms of the dimension $d$, and also establish the \Rnorm-suboptimality of ridge function interpolants.
(In Section~\ref{sec:cosine-approx}, we extend the averaging strategy to other types of target functions, expanding the scope of our structural findings.)

\begin{itheorem}[Min-$\mathcal{R}$-norm interpolation is sub-optimal for learning parities]
Suppose the dataset $\{ (\bx_i,\chi(\bx_i)) \}_{i\in[n]} \subset \flip^d \times \flip$ used in \eqref{interpolating-r-norm} is an i.i.d.~sample, where $\bx_i \sim \unif(\flip^d)$ is labeled by the $d$-variable parity function $\chi$ for all $i \in [n]$.
  If the sample size is $n = o(d^2/\sqrt{\log d})$, then with probability at least $1/2$, every solution to \eqref{interpolating-r-norm} has mean squared error at least $1-o(1)$ for predicting $\chi$ over $\unif(\flip^d)$.
\end{itheorem}
This result is presented formally in Section~\ref{ssec:parity-gen-lb}, and it is complemented by a sample complexity upper bound in Section~\ref{ssec:parity-gen-ub}.
The results are stated for the parity function on all $d$ variables, but the same holds for any parity function over $\Omega(d)$ variables.
It is well-known that an i.i.d.~sample of size $O(d)$ is sufficient for learning parity functions exactly ~\citep{helmbold1992learning,fischer1992learning}, and hence
we conclude that the \Rnorm inductive bias is insufficient for achieving the statistically optimal sample complexity for this learning problem.

\subsection{Related work}
\label{ssec:related}

\paragraph{Variational norms and inductive biases of optimization methods.}
Many variational norms (such as \Rnorm) from functional analysis
can be regarded as representational costs that induce topologies on the space of infinitely-wide neural networks with certain activation functions.
Prior works have analyzed these norms for homogeneous activation functions like ReLU~\citep[e.g.,][]{k01,m04,b17,sess19,owss19}; see \citet{sx21} and references therein for a comparison.
In particular, the work of \citet{owss19}
provided analytical descriptions of \Rnorm in terms of the Radon transform of the function itself.
This work was extended to higher powers of ReLU by \citet{pn21a}.

The variational norms are also connected to the implicit biases of optimization methods for training neural networks.
In the context of univariate functions, the dynamics of gradient descent was shown to be biased towards (adaptive) linear or cubic spline depending on the optimization regime~\citep{wtspzb19, skm21, hbg18}, and these results have been partially extended to the multivariate case~\citep{jm20}.
For classification problems, the implicit bias of gradient descent was connected to a variational problem related to \Rnorm with margin constraints on the data~\citep{bc21}.

\paragraph{Solutions to the variational problem.}
\citet{dduf21} and \citet{hanin21} fully characterized the form of all solutions of \eqref{interpolating-r-norm} for one-dimensional datasets (as discussed above).
However, pinning down even a single solution for general multidimensional datasets appears to be difficult; \citet{ep21} was able to do so for rank-one datasets, where all the feature vectors lie on a line.
The datasets we study do not satisfy the rank-one condition of \citet{ep21}, and thus we require different techniques to analyze multi-directional functions.

\paragraph{Adaptivity.}
In the context of non-parametric regression, it is well-known that (deep) neural networks achieve minimax-optimal rates in the presence of low-dimensional structure in the target function \citep[e.g.,][]{s20, bk19, kk05, g02}.
The convergence rates in these works depend only on the intrinsic dimension of the target function (and not the ambient dimension) and are achieved by optimally trading off accuracy and regularization in certain deep neural network architectures.
Recent works \citep{kb16, pn21b, b17} consider two-layer neural networks with variational norm (similar to \Rnorm) regularization, which also allows for adaptivity to low-dimensional structure.
That is, a function $g \colon \R^d \to \R$ depending only on a $k$-dimensional projection of its input $x$, i.e., $g(x) = \phi(Ux)$ for some $U \in \R^{k \times d}$ (with orthonormal rows) and $\phi \colon \R^k \to \R$ has variational norm no greater than that of the corresponding low-dimensional function $\phi$~\citep{b17}. 
In particular, \citet{b17} and \citet{kb16} studied minimax rates under ridge target functions where $k=1$. Our results on generalization are of a different flavor: 
rather than striking a careful balance between fitting and regularization to achieve minimax rates, we study the behavior of \Rnorm-minimizing interpolation.

Regularization based on weight decay (equivalent to \Rnorm for shallow networks) has also been used to obtain minimax rates for learning smooth target functions.
\citet{pn21b} do so by drawing analogies to spline theory, while \citet{wl21} consider a connection to the Group Lasso. 
\citet{zw22} exploits depth to promote stronger sparsity regularizes. 
This is distinct from the low-dimensional structures studied in this work and mentioned above.

\paragraph{Learning ridge functions and parity functions with neural nets.} 
Target functions that depend on low-dimensional projections of the input (of which ridge functions are the simplest case) have been long studied in statistics \citep[see, e.g.,][]{li18}, and learning such functions
is one of the simplest problems where neural network training demonstrates adaptivity.
Such demonstrations typically require going beyond the neural tangent kernel regime and have been used to explain the ``feature learning'' ability of neural networks \citep{frei2022random, damian2022neural, mousavi2022neural, bbss22}.
Several recent works have considered the prospects of learning (sparse) parity functions by training neural nets with gradient-based algorithms~\citep{as20, dm20, mkas21, begkmz22, telgarsky2022feature}.
The positive results express parities as low-weight linear combinations of (random) ReLUs, which motivates our focus on the variational norm of approximating neural nets.
Our sample complexity lower bound shows that, even if computational and optimization considerations are set aside, the inductive bias imposed by the \Rnorm may lead to suboptimal statistical performance.

\paragraph{Averaging and ensembling.} Neural networks have been interpreted as forms of averaging or ensemble methods to explain their statistical properties~\citep[e.g.,][]{bartlett1996valid,baldi2013understanding,gal2016dropout,olson2018modern}.
Our use of averaging is different in that it serves as an approximation-theoretic mechanism for achieving smaller \Rnorm.

\paragraph{Weight lower bounds for other explicit functions.} Representation costs for two-layer neural networks to approximate other explicit functions have been explored in several prior works~\citep{mcpz13,dan17,ss17,ses19}.
These works establish exponential lower-bounds on the width of two-layer networks needed to approximate functions that are represented more compactly by three-layer networks.
These results also imply lower-bounds on the size of second-layer weights in a two-layer network after fixing the width of the network.
In contrast, our results hold for networks of unbounded width and for a target function that can be exactly represented by a two-layer networks of $\poly(d)$ width.

 \section{Preliminaries}\label{sec:prelims}

\subsection{Notation}\label{ssec:notation}

In this work, we consider real-valued functions over the radius-$\sqrt{d}$ Euclidean ball $\domain := \setl{ x \in \R^d : \norml[2]{x} \leq \sqrt{d} }$.
Let $\parity_S \colon \domain \to \flip$ denote the multi-linear monomial $\parity_S(x) := \prod_{i \in S} x_i$ over variables indexed by $S \subseteq [d]$, and let $\parity := \parity_{[d]}$.
On input $x \in \flip^d$, $\parity_S(x)$ computes the \emph{parity} of $\{ x_i : i \in S \}$.
We say $g \colon \domain \to \R$ is a \emph{ridge function} if $g(x) = \phi(v^\T x)$ for some unit vector $v \in \sph$ and function $\phi \colon [-\sqrt{d},\sqrt{d}] \to \R$.
A function $\phi$ is \emph{$\rho$-periodic} if $\phi(z+\rho)=\phi(z)$ for all $z \in \R$.

We consider two-layer neural networks (of infinite and finite width) with ReLU activations $\relu(z) := \max\setl{0, z}$.
Let $\measures$ denote the space of signed measures over $\sph \times [-\sqrt{d},\sqrt{d}]$.
For $\mu \in \measures$, let $g_\mu: \domain \to \R$ denote the infinite-width neural network given by
\[g_\mu(x) := \int_{\sph \times [-\sqrt{d},\sqrt{d}]} \relu(w^\T x + b) \, \mu(\dd w, \dd b).\]
The total variation norm of $\mu$ is $\abs{\mu} := \int_{\sph \times [-\sqrt{d},\sqrt{d}]} \abs{\mu}(\dd w, \dd b)$, where $\abs{\mu}(\dd w,\dd b)$ is the corresponding total variation measure (somewhat abusing notation).
The width-$m$ neural network $g_{\theta}$ with parameters $\theta = (a^{(j)}, w^{(j)}, b^{(j)})_{j \in [m]} \in (\R \times \sph \times [-\sqrt{d},\sqrt{d}])^m$ is given by
\[ g_\theta(x) := \sum_{j=1}^m a^{(j)} \relu(w^{(j) \T} x + b^{(j)}). \]
We regard $g_\theta$ as an infinite-width neural network with the ``sparse'' atomic measure
\begin{equation*}
  \mu_\theta = \sum_{j=1}^m  a^{(j)} \delta_{(w^{(j)},b^{(j)})} .
\end{equation*}
Observe that $g_\theta = g_{\mu_\theta}$ and $\abs{\mu_\theta} = \sum_{j=1}^m \absl{a^{(j)}} = \norml[1]{a}$.

Our constructions frequently use \textit{sawtooth functions}, a family of ridge functions that are composed of $t+1$ repetitions of a triangular wave that draw inspiration from a construction of \citet[Proposition 4.2]{ys19}.
For $t \in \{0, \dots d\}$ with $t \equiv d \pmod2$ and $w \in \flip^d$, let $s_{w,t}(x) := \smash{(-1)^{d-t} \parity(w) \phi_t(w^\T x)}$ where $\phi_t: \R \to \R$ is a function that forms a piecewise affine interpolation between the points $(-t-1, 0)$, $\smash{\{(t - 2\tau, (-1)^\tau)\}_{\tau \in \{0, \dotsc, t\}}}$, $(t+1, 0)$, and $\phi_t(z) = 0$ for all $z \leq -t-1$ and $z \geq t+1$.
We refer to $t$ as the \emph{width} of the sawtooth function $s_{w,t}$.
Note that $s_{w,t}$ is $\sqrt{d}$-Lipschitz and $s_{w,t}(x) = \parity(x) \indicator{\abs{w^\T x} \leq t}$ for all $x \in \flip^d$.
Also, $s_{w,t}$ can be expressed as a neural network $g_\theta$ with width $m \leq O(t+1)$ and $\absl{a^{(i)}} \leq O(\sqrt{d})$ for each $i \in [m]$.

Let $\nu := \unif(\flip^d)$ denote the uniform distribution on $\flip^d$, and let $\nun$ denote the empirical distribution on $\bx_1,\dotsc,\bx_n \simiid \nu$.
We use the following inner products and norms over the vector space of real-valued functions on $\flip^d$ with respect to a distribution $\nu_0$ (such as $\nu$ or $\nun$):
\[\innerprodnuzero{g}{h} := \expectationl[\bx \sim \nu_0]{g(\bx)h(\bx)}, \quad \normnuzero{g} := \innerprodnuzero{g}{g}^{1/2}, \quad \normnuzeroinf{g} := \max_{x \in \support(\nu_0)} \absl{g(x)}.\]

\subsection{\Rnorm and attainment of the infimum}
\label{ssec:rnorm}

We now recall the definition of the \Rnorm of a function $g \colon \domain \to \R$, presented here in a variational form as the minimum cost of representing $g$ as an infinite-width neural network with a ``skip-connection'':
\begin{equation}
    \label{r-norm}
    \tag{$\calR$-norm}
    \rnorm{g} := \inf_{\mu \in \measures, v \in \R^d, c \in \R} \abs{\mu} \quad \text{s.t.} \quad g(x) = g_{\mu}(x) + v^\T x + c \ \ \forall x \in \domain.
\end{equation}
Indeed, $\rnorm{\cdot}$ is a semi-norm on the space of functions with finite \Rnorm.
It was initially introduced by \citet{owss19} along with explicit characterizations in terms of the Radon transform.
See the works of \citet{owss19}, \citet{pn21a}, and \citet{sx21} for more discussion about the \Rnorm and its connections to other function spaces.

The appearance of the affine component $v^\T x + c$ in the definition of \Rnorm has implications for how the bias terms are treated.
Notice that a neuron $x \mapsto \relu(w^\T x + b)$ with bias $\abs{b} \ge \sqrt{d}$ behaves as an affine function over the domain of interest $\domain$, so it can be absorbed into the ``free'' affine component (in the definition of \Rnorm) so as to not be counted towards the \Rnorm.
Other works \citep[e.g.,][]{sx21} consider a different variational norm, $\vnorm{\cdot}$, which does not have ``free'' affine components, but instead permits biases $b$ to be in the larger range $[-2\sqrt{d}, 2\sqrt{d}]$.
These differences in the way affine components are accommodated do not lead to different function spaces~\citep[see][Theorem~6]{pn21b}, and the results of this paper for \Rnorm also hold for these other variational norms, as we demonstrate in Appendix~\ref{asec:vnorm}.

Although the \Rnorm is defined in terms of an infimum,
it has been shown by \citet[Lemma 2; see also Proposition~\ref{prop:integral-representation} in Appendix~\ref{asec:rnorm}]{pn21b} that
the infimum is always achieved by a particular signed measure $\mu \in \measures$.
Since the total variation norm is sparsity-inducing, the objective in \eqref{interpolating-r-norm} favors finite-width networks.
It can be shown, using an extension of Caratheodory's theorem~\citep{rosset2007}, that \eqref{interpolating-r-norm} in fact always has a finite-width solution.
That is, \eqref{interpolating-r-norm} is solved by the sum of an affine function $x \mapsto v^\T x + c$ and a width-$m$ neural network, for some $m \leq \max\{0, n-(d+1)\}$.
This claim is formalized as Theorem~\ref{thm:sparse-vp-solution} and proved in Appendix~\ref{asec:rnorm}.
Thus, considering finite-width neural networks is sufficient to determine the value of \eqref{interpolating-r-norm}.\footnote{We note that the finite-width solution to \eqref{interpolating-r-norm} is not necessarily unique; \citet{hanin21} discusses this issue in the one-dimensional case ($d=1$) under general data models.}

The following lemma, which is a minor elaboration on Lemma 25 of \citet{pn21a}, relates the \Rnorm of a finite-width network to the $\ell_1$-norm of its top-layer weights.

\begin{lemma}\label{lemma:rnorm-relu-l1} 
Let $v \in \R^d$, $c \in \R$, and $\theta =  (a^{(j)}, w^{(j)}, b^{(j)})_{j \in [m]} \in (\mathbb{R} \times \sph \times [-\sqrt{d},\sqrt{d}])^m$ be the set of parameters of a finite neural network where $(w^{(i)},b^{(i)}) \neq (w^{(j)},b^{(j)})$ for all $i \neq j$. 
\begin{enumerate}[label=(\roman*)]
    \item The \Rnorm of the sum of $g_\theta$ and an affine function $v^\T x + c$ satisfies 
    \begin{equation}
        \label{eq:rnorm-relu-l1}
        \rnorm{g_{\theta}(x) + v^\T x + c} \le \norm[1]{a} =
        \absl{a^{(1)}} + \dotsb + \absl{a^{(m)}} .
\end{equation}
    \item Moreover, if the measure $\mu_\theta$ is even in a distributional sense (i.e., $\mu_{\theta}(w,b) = \mu_{\theta}(-w,-b)$), then the inequality in \eqref{eq:rnorm-relu-l1} holds with equality. \end{enumerate}
\end{lemma}

Note that our assumption that $\mu_\theta$ is even in Lemma~\ref{lemma:rnorm-relu-l1}(ii) precludes the case where $a^{(i)} = -a^{(j)}$ and $(w^{(i)},b^{(i)}) = (-w^{(j)},-b^{(j)})$ for some $i \neq j$.
This is needed because if such a case were allowed, we would have $a^{(i)} \relu(w^{(i) \T} x + b^{(i)}) + a^{(j)} \relu(w^{(j) \T} x + b_j) = a^{(i)} (w^{(i) \T} x + b^{(i)})$ for all $x \in \domain$---an affine function.
After ruling out these cases, we can apply the argument of \citet{pn21a} to prove Lemma~\ref{lemma:rnorm-relu-l1}(ii). 

\subsection{\Rnorm of ridge functions}
\label{sec:rnorm-ridge}

Prior works illuminate precise formulations of the \Rnorm, and characterize solutions to \eqref{interpolating-r-norm}, albeit only for the one-dimensional setting \citep{hanin21, sess19, ep21}.
These results are nevertheless useful for analyzing ridge functions in $d$-dimensional space.

\begin{theorem}
\label{thm:single-index-r-norm}
For any ridge function $g \colon \domain \to \R$ of the form $g(x) = \phi(w^\T x)$ where $w \in \sph$ and $\phi \colon [-\sqrt{d}, \sqrt{d}] \to \R$ is Lipschitz, we have
$$ \rnorm{g} = \tv{\phi'} \coloneqq \esssup_{- \sqrt{d} \leq t_0 < t_1 < \dotsb < t_r \leq \sqrt{d}; \, r \in \mathbb{N}} \sum_{i=1}^{r} \abs{\phi'(t_{i}) - \phi'(t_{i-1})} , $$
where $\phi'$ is a right continuous derivative of $\phi$.\footnote{Take $\phi'(u) = \lim_{t \downarrow 0}{\frac{\phi(u+t) - \phi(u)}{t} }$; the limit exists almost everywhere by Rademacher's theorem.}
\end{theorem}
\begin{remark}
If $\phi$ is twice differentiable, then $\rnorm{g} = \int_{-\sqrt{d}}^{\sqrt{d}}{ |\phi''(u)|\dd u } = \norm[1]{\phi''}$.
Intuitively, this $\ell_1$-norm penalty induces sparsity in the second derivative, leading to representations that use few neurons.
In contrast, minimizing the $\ell_2$-norm penalty $\norm[2]{\phi''}$ on the second derivative yields a cubic spline~\citep{kw71}.
\end{remark}
\begin{proof}
  Without loss of generality, $g$ only depends on the first coordinate $x_{1}$ due to the invariance of the \Rnorm to rotation \citep[cf.~Proposition 11 of][]{owss19}. The result then follows from Remark~4 of \citet{pn21b}.
\end{proof}

This bound on the \Rnorm for ridge functions (and univariate functions) is critical for analyses of the solutions to \eqref{interpolating-r-norm} for $d = 1$ \citep{hanin21, sess19}.
It suggests a potential approach for our high-dimensional setting: project the dataset to every one-dimensional subspace, interpolate the data with a ridge function that points in that direction, directly compute the \Rnorm of each using Theorem~\ref{thm:single-index-r-norm}, and return the ridge function with the lowest \Rnorm.
In the sequel, we examine the optimality of this approach, and find that ridge functions \textit{cannot} be optimal solutions to \eqref{interpolating-r-norm}, even when the dataset can be perfectly fit by a ridge function.

 \section{Solutions to the variational problem for parity are multi-directional}\label{sec:parity-approx}

In this section, we study the \Rnorm of neural networks that solve \ref{interpolating-r-norm} or \ref{approximate-r-norm} for
the (full) \emph{parity dataset} $\{ (x,\parity(x)) : x \in \flip^d \}$, which has size $n = 2^d$.
For simplicity, the labels are provided by the parity function $\parity$ over all $d$ variables, although the same quantitative results (up to constant factor differences) hold for any $\parity_S$ with $|S| = \Theta(d)$.

The high level message is that, despite the fact that
this dataset
can be exactly fit using ridge functions, the solutions to \eqref{interpolating-r-norm} and \eqref{approximate-r-norm} are \emph{not} ridge functions and instead must be multi-directional.

\subsection{Every ridge parity interpolant has \Rnorm $\Omega(d^{3/2})$}\label{ssec:parity-rnorm-ridge-approx-lb}

We first show that any $\epsilon$-approximate interpolant of the parity dataset \emph{that is also a ridge function} must have \Rnorm $\Omega(d^{3/2})$.
This lower-bound is established even for $\epsilon=1/2$.

\begin{restatable}{theorem}{thmridgeparitylb}\label{thm:ridge-parity-lb}
For $d\geq2$, let $\ridge_d$ be the set of functions $g \colon \domain \to \R$ such that $g(x) = \phi(w^\T x)$  for some $w \in \sph$ and Lipschitz continuous $\phi \colon [-\sqrt{d},\sqrt{d}] \to \R$.
Then
\[ \inf\setl{ \rnorm{g} : g \in \operatorname{Ridge}_d , \, \cubenorminfty{g - \chi} \leq 1/2 } \geq d^{3/2}/(2\sqrt2) . \]
\end{restatable}

The proof
constructs a labeled dataset of $d+1$ points, and shows that any ridge function $g(x) = \phi(w^\T x)$ that approximates that dataset must have many high-magnitude oscillations.
These oscillations imply a lower bound on $\tv{\phi'}$, which proves the claim by way of Theorem~\ref{thm:single-index-r-norm}.

\begin{proof}
Take any $g \in \ridge_d$ of the form \smash{$g(x) = \phi(w^\T x)$} for some function $\phi$ and vector $w$ satisfying the approximation constraint \smash{$\cubenorminfty{g-\parity} \leq 1/2$}.
Suppose for sake of contradiction that $w_i = 0$ for some $i \in [d]$.
Then, there exists a pair of points $x, x' \in \{-1, 1\}^d$ that are identical except in the $i$-th positions, $x_i$ and $x'_i$.
Thus, $\parity(x) = -\parity(x')$, but $w^\T x = w^\T x'$ and hence $g(x) = g(x')$; this contradicts the approximation constraint.
So, we may henceforth assume that $w_i \neq 0$ for all $i\in [d]$.

For each $i \in \{0,1,\dotsc,d\}$, define 
\[
    x^{(i)} := (\signl{w_1}, \dots, \signl{w_i}, -\signl{w_{i+1}}, \dots, -\signl{w_d}) .
\]
Because the parity of $x^{(i)}$ alternates with $i$, i.e., $\parity(x^{(i)}) \neq \parity(x^{(i+1)})$, $\signl{g(x^{(i)})}$ also alternates because $g$ satisfies the approximation constraint.
Furthermore, again due to the approximation constraint, we have $\absl{g(x^{(i)}) - g(x^{(i+1)})} \geq 1$.
We claim that, because $\phi$ interpolates $d+1$ well-separated data points $(w^\T x^{(i)}, \phi(w^\T x^{(i)}))$ that satisfy $w^\T x^{(i)} < w^\T x^{(i+1)}$ for all $i \in \{0,1,\dotsc,d-1\}$,
there must be a large cost for representing $\phi$ using a neural network.
By Theorem~\ref{thm:single-index-r-norm}, it suffices to obtain a lower bound on $\tv{\phi'}$, since this will imply a lower bound on $\rnorm{g}$.

By Lemma~\ref{lemma:calc} (in Appendix~\ref{assec:parity-rnorm-ridge-approx-lb}; essentially the mean value theorem), for every $i \in \{0,1,\dotsc,d-1\}$, there exists $A_i \subseteq [w^\T x^{(i)}, w^\T x^{(i+1)}]$ with Lebesgue measure $\Leb(A_i)>0$ such that, for every $z^{(i)} \in  A_i$, we have
\[ \absl{\phi'(z^{(i)})} 
\geq \frac12 \cdot \frac{\absl{\phi(w^\T x^{(i+1)}) - \phi(w^\T x^{(i)})}}{w^\T x^{(i+1)} - w^\T x^{(i)}},\]
and $\smash{\sign{\phi'(z^{(i)})} = \sign{\phi(w^\T x^{(i+1)}) - \phi(w^\T x^{(i)})}}$.
The fact that the signs of $\phi(w^\T x^{(i)})$ alternate with $i$ implies that the signs of  $\phi'(z^{(i)})$ also alternate with $i$.
We now lower-bound the total variation of $\phi'$ using the fact that $\prod_{i=1}^{d}\Leb(A_i) > 0$ and taking advantage of the alternating signs:
\begin{align*}
2\tv{\phi'} &= 2 \esssup_{- \sqrt{d} \leq t_0 < t_1 < \dotsb < t_r \leq \sqrt{d}; \, r \in \mathbb{N}} \sum_{i=1}^{r} \abs{\phi'(t_{i}) - \phi'(t_{i-1})} \\
&\geq 2\sum_{i=1}^{d-1} \absl{\phi'(z^{(i)}) - \phi'(z^{(i-1)})}
= 2\sum_{i=1}^{d-1}\parenl{ \absl{\phi'(z^{(i)})} +  \absl{\phi'(z^{(i-1)})}}
\geq 2 \sum_{i=0}^{d-1} \absl{\phi'(z^{(i)})}\\
&\geq \sum_{i=0}^{d-1} \frac{\absl{\phi(w^\T x^{(i+1)}) - \phi(w^\T x^{(i)})}}{w^\T x^{(i+1)} - w^\T x^{(i)}} 
\geq \sum_{i=0}^{d-1}  \frac{1}{w^\T x^{(i+1)} - w^\T x^{(i)}} \\
&\ge  \frac{d^2}{\sum_{i=0}^{d-1} w^\T x^{(i+1)} - w^\T x^{(i)}}
= \frac{ d^2}{ w^\T x^{(d)} - w^\T x^{(0)}} 
\ge \frac{d^2}{\norml[2]{w}\norml[2]{x^{(d)}-x^{(0)}}}
= \frac{d^{3/2}}{\sqrt{2}} .
\end{align*}
The second-to-last inequality is a consequence of Cauchy-Schwarz: for any $a_1, \dots, a_d > 0$, $d^2 = \smash{(\sum_i \sqrt{a_i} / \sqrt{a_i})^2 \leq (\sum_i a_i) (\sum_i 1/a_i)}$.
Therefore, $\rnorm{g} = \tv{\phi'} \geq d^{3/2}/(2\sqrt2)$.
\end{proof}

The lower-bound in Theorem~\ref{thm:ridge-parity-lb} is tight up to constants, because the sawtooth function
$s_{\vec{1},d}$ satisfies the constraints of \eqref{interpolating-r-norm} and has $\rnorm{s_{\vec{1},d}} = O(d^{3/2})$.
 \subsection{Existence of a multi-directional parity interpolant with \Rnorm $O(d)$}\label{ssec:parity-rnorm-ub}

We now show that the $\Omega(d^{3/2})$ \Rnorm lower-bound from Theorem~\ref{thm:ridge-parity-lb} for ridge functions can be avoided by neural networks that are not ridge functions.
The main idea is to employ an \emph{averaging strategy} that combines a collection of distinct ridge functions, each of which perfectly fits a small fraction of the parity dataset---those on the ``equator'' relative to the ridge direction---while ignoring the ``outliers'' in that direction.
Because all points on the cube are ``outliers'' for some directions and on the ``equator'' for others, this strategy ultimately ensures that every example is perfectly fit.

\begin{theorem}\label{thm:parity-rnorm-ub}
For any even\footnote{Our results also hold for odd $d$, but the proofs are more tedious.} $d$, there exists a neural network $g: \domain \to \R$ having $g(x) = \parity(x)$ for all $x \in \flip^d$ such that $\rnorm{g} \leq O(d)$. 
\end{theorem}

\begin{proof}
Recall that the sawtooth function $s_{w, 0}: \domain \to \R$ satisfies \smash{$s_{w, 0}(x) = \parity(x) \indicator{w^\T x = 0}$} for all \smash{$x \in \flip^d$}. 
By construction, $s_{w,0}$ is a ridge function that is a single ``bump'' around zero in the direction of $w$,
and $\rnorml{s_{w, 0}} \leq \smash{O(\sqrt{d})}$.
Consider \smash{$\bw \sim \unif(\flip^d)$}. 
By symmetry, $\pr{\bw^\T x = 0} = \pr{\bw^\T x' = 0}$ for all $x, x' \in \flip^d$, so 
\[
  \EE{s_{\bw, 0}(x)} = \parity(x) \cdot \pr{\bw^\T x = 0} = \parity(x) \cdot 2^{-d} \cdot \absl{\{v \in \flip^d : v^\T x = 0\}} = \parity(x) \cdot q ,
\]  where \smash{$q := \binom{d}{d/2} / 2^d = \Theta(1/\sqrt{d})$}.
Define \smash{$g(x) := \frac{1}{q 2^d} \sum_{w \in \flip^d} s_{w, 0}(x)$}. 
Then $g(x) = \frac{1}{q} \EE{s_{\bw, 0}(x)} =  \parity(x)$ for each $x \in \flip^d$, i.e., $g$ interpolates the parity dataset.
Finally, we bound the \Rnorm: \[\rnorm{g} \leq \frac1{q2^d} \sum_{w \in \flip^d} \rnorm{s_{w,0}} \leq \frac1q \cdot O(\sqrt{d}) \leq O(d).\qedhere\]
\end{proof}

While Theorem~\ref{thm:parity-rnorm-ub} successfully exhibits a neural network $g$ that perfectly fits the parity dataset with $\rnorm{g} = O(d)$, the width of $g$ is $\Omega(2^d)$.
We next show that by allowing non-zero $\Lpu{\infty}$ error in the approximation, we can achieve a construction with both $O(d)$ \Rnorm and $\poly(d)$ width.

\begin{restatable}{theorem}{thmparityrnormapproxub}
\label{thm:parity-rnorm-approx-ub}
There is a universal constant $c>0$ such that the following holds.
For any even $d$, any $\epsilon \in (0,1)$, and any even $t \in \setl{0, 2,\dotsc,d}$,
there exists a function $g \colon \domain \to \R$ that can be represented by a width-$m$ neural network such that $\cubenorminfty{g - \parity} \leq \epsilon$, where
\begin{align*}
  m & \leq O(d^{3/2} \sqrt{\log(1/\epsilon)}/\epsilon^2) & \text{and} \quad \rnorm{g} & \leq O(d \log(1/\epsilon)) && \text{if $t \leq c\sqrt{d \log(1/\epsilon)}$} ; \\
  m & \leq O(d^2/(\epsilon t)) & \text{and} \quad \rnorm{g} & \leq O(t\sqrt{d}) && \text{otherwise} .
\end{align*}
Moreover, $g$ can be expressed as a linear combination of width-$t$ sawtooth functions.\end{restatable}

\begin{remark}
Suppose $\epsilon$ is a constant.
Using $t = \Theta(d)$, we obtain a neural network of width $m = O(d)$ and $\rnorm{g} = O(d^{3/2})$, matching the properties of
the sawtooth (ridge function) interpolant $s_{w, d}$.
Using $t = \Theta(1)$, we obtain a neural network of width $m = O(d^{3/2})$ and $\rnorm{g} = O(d)$, matching the properties of
the interpolant from Theorem~\ref{thm:parity-rnorm-ub} but with almost exponentially smaller width.
\end{remark}

A more detailed version of Theorem~\ref{thm:parity-rnorm-approx-ub} (which also specifies the \emph{intrinsic dimensionality} of $g$) is stated and proved in Appendix~\ref{assec:parity-rnorm-ub}.
The proof uses a similar
technique as that of Theorem~\ref{thm:parity-rnorm-ub}, but
instead averages randomly sampled sawtooth functions \smash{$s_{\bw^{(1)}, t}, \dots, s_{\bw^{(k)}, t}$} for \smash{$\bw^{(j)} \sim \unif(\flip^d)$} of width $t$.
We show that for sufficiently large $k$, every $x \in \flip^d$ lies the in the ``active'' region of about the same number of sawtooth functions; this yields a good approximation of $\parity(x)$ for all $x$.
 \subsection{Every parity interpolant has \Rnorm $\Omega(d)$}\label{ssec:parity-rnorm-approx-lb}

Finally, we show that \Rnorm upper-bounds from Theorems~\ref{thm:parity-rnorm-ub} and~\ref{thm:parity-rnorm-approx-ub} are tight.
That is, we show that every solution to \eqref{approximate-r-norm} for the parity dataset has \Rnorm $\Omega(d)$, even for constant $\epsilon$.
This is implied by the following stronger result, which requires only $\Lpu{2}$ approximation, as opposed to $\Lpu{\infty}$.

\begin{restatable}{theorem}{thmparityrnormlb}\label{thm:parity-rnorm-lb}
For any $d \geq 8$ and $\alpha \in (0, 1)$, 
$\inf\setl{\rnorm{g}: \cubenormtwo{g - \parity} \leq 1- \alpha} \geq \alpha d/8$.
\end{restatable}

The core of the proof of Theorem~\ref{thm:parity-rnorm-lb} (given in Appendix~\ref{assec:parity-rnorm-approx-lb}) is an upper-bound on the correlation of any fixed ReLU neuron with the parity function $\parity$.

We note that a result analogous to Theorem~\ref{thm:parity-rnorm-lb} also holds for most \emph{sampled parity datasets} (defined in Section~\ref{sec:parity-gen}).
This result is stated and proved in Appendix~\ref{asec:parity-approx-sampled}.

  \section{Generalization properties of solutions to the variational problem}\label{sec:parity-gen}

In this section, we consider the generalization properties of a learning algorithm that returns a solution to \eqref{interpolating-r-norm} for a \emph{sampled parity dataset} $\{ (\bx_i,\parity(\bx_i)) : i \in [n] \}$ for $\bx_1,\dotsc,\bx_n \simiid \nu$.
(Again, for simplicity, we label data using $\parity$, but the same results hold for any $\parity_S$ with $|S| = \Theta(d)$.)

We show that $n = o(d^2/\sqrt{\log d})$ results in a predictor with nearly trivial accuracy.
Note that information-theoretically, $n \geq O(d)$ is sufficient for learning any parity function~\citep{helmbold1992learning,fischer1992learning}.
This means that the inductive bias based on \Rnorm is not sufficient to achieve statistically optimal sample complexity for learning parity functions.

\subsection{Poor generalization with $n \ll d^2/\sqrt{\log d}$ samples}
\label{ssec:parity-gen-lb}

We first give a lower bound on the sample size needed for non-trivial generalization for learning parity functions by solving \eqref{interpolating-r-norm} with the sampled parity dataset.

\begin{theorem}
  \label{thm:parity-gen-lb}
  If $n = o(d^2/\sqrt{\log d})$, then with probability at least $1/2$, every solution $\bg \colon \domain \to \R$ to \eqref{interpolating-r-norm} for the sampled parity dataset has $\cubenormtwo{\bg - \parity} \geq 1 - o(1)$.
\end{theorem}

Its proof relies on the following approximation lemma, which shows the existence of a low-\Rnorm network $\bg$ that perfectly fits all $n$ samples. 
The lemma (which is proved in Appendix~\ref{assec:parity-gen-lb}) defines $\bg$ with the same ``cap construction'' used in Theorem~1 of \citet{bln20}. 
\begin{restatable}{lemma}{lemmasmallsampleparityapprox}\label{lemma:small-sample-parity-approx}
  There is an absolute constant $c > 0$ such that the following holds.
If $n \leq cd^2$,
and $\bx_1, \dotsc \bx_n \simiid \nu$, then with probability at least $1/2$, there exists $\bg \colon \domain \to \R$ with $\bg(\bx_i) = \parity(\bx_i)$ 
for all $i \in [n]$ and $\rnorm{\bg} \leq \smash{4n\sqrt{\ln d}/d}$.
\end{restatable}
We conclude that generalization fails in this low-sample regime because Theorem~\ref{thm:parity-rnorm-lb} shows that no network with sufficiently small \Rnorm can correlate with parity.

\begin{proof}[Proof of Theorem~\ref{thm:parity-gen-lb}]
  Let $\alpha := \smash{64n\sqrt{\ln d}/d^2}$, so $\alpha = o(1)$ by assumption on $n$.
  By Theorem~\ref{thm:parity-rnorm-lb}, every $g \colon \domain \to \R$ with \smash{$\cubenormtwo{g - \parity} \leq 1 - \alpha$} has \smash{$\rnorm{g} \geq \alpha d / 8 \geq 8n\sqrt{\ln d}/d$}.
  However, by Lemma~\ref{lemma:small-sample-parity-approx}, 
with probability at least $1/2$, every solution $\bg$ to \eqref{interpolating-r-norm}  for the dataset $\setl{ (\bx_i,\parity(\bx_i)) }_{i \in [n]}$ has $\rnorm{\bg} \leq \smash{4n\sqrt{\ln d}/d}$.
  In this event, the solutions $\bg$ have $\cubenormtwo{\bg - \parity} \geq 1 - \alpha = 1 - o(1)$.
\end{proof}

\subsection{Good generalization with $n \gtrsim d^3$ samples}
\label{ssec:parity-gen-ub}

We complement the lower-bound in Theorem~\ref{thm:parity-gen-lb} with the following sample complexity upper-bound.

\begin{restatable}{theorem}{thmparitygenub}\label{thm:parity-gen-ub}
  There is an absolute constant $C>0$ such that the following holds.
  For any $\epsilon \in (0,1)$ and $\delta \in (0,1)$, if
  \smash{$n \geq C(\log(1/\delta) + d^3/\epsilon^2)$},
then
with probability at least $1-\delta$, every solution $\bg \colon \domain \to \R$ to \eqref{interpolating-r-norm} for the sampled parity dataset satisfies
  $\cubenormtwo{\parity - \clip \circ\, \bg}^2 \leq \epsilon$,
where $\clip(t) := \min\{\max\{t, -1\}, 1\}$.
\end{restatable}

For technical reasons, we only bound the $\Lpu{2}$ error of the natural truncation of a solution to \eqref{interpolating-r-norm}.
The proof in Appendix~\ref{assec:parity-gen-ub} is based on standard Rademacher complexity arguments.

We note that there is a gap between our lower bound (Theorem~\ref{thm:parity-gen-lb}) and upper bound (Theorem~\ref{thm:parity-gen-ub}): roughly a factor of $d\sqrt{\log d}$.
We believe that this gap could be narrowed if one resolves the open question raised by \citet{bln20} about the minimum Lipschitz constant achievable by two-layer ReLU networks of width $m$ networks that interpolate a sample of size $n$; Lemma~\ref{lemma:small-sample-parity-approx} is derived from a theorem that produces networks with smoothness conjectured to be sub-optimal.
Nevertheless, our lower bound in Theorem~\ref{thm:parity-gen-lb} is already high enough to establish the statistical suboptimality of solutions to \eqref{interpolating-r-norm}.
 \section{Generality of the averaging technique for minimizing \Rnorm}\label{sec:cosine-approx}

In this section, we show how the benefit of averaging goes beyond the parity dataset.
We consider an \emph{$f$-dataset} $\{(x, f(x)\}_{x \in \flip^d}$, a generalization of the parity dataset where $f(x) = \phi(v^\T x)$ is a ridge function with $L$-Lipschitz and $\rho$-periodic $\phi$.
For another dataset generated by oscillatory ridge functions, we prove the same contrast between minimum-\Rnorm interpolation with and without ridge constraints, so long as the periodicity $\rho$ is not too small (specifically, $\rho \geq 1/\sqrt{d}$).
More concretely,
suppose the dataset $\{ (x_i,f(x_i)) \}_{i\in[n]} \subset \flip^d \times \flip$ used in \eqref{interpolating-r-norm} and \eqref{approximate-r-norm} is the \emph{$f$-dataset}, where $v \in \flipd^d$ and $\phi$ is $\rho$-periodic and $\frac1\rho$-Lipschitz.
  Then we have the following:
  \begin{itemize}
    \item  The optimal value of \eqref{approximate-r-norm} for constant $\epsilon \in (0,1/2)$ is $\tilde{O}(\sqrt{d} / \rho)$. (Theorem~\ref{thm:periodic-ub}) 
    \item The optimal value of \eqref{approximate-r-norm} for constant $\epsilon \in (0,1/2)$---with the additional constraint that $g$ be a ridge function---is $\Omega(\sqrt{d} / \rho^2)$. (Theorem~\ref{thm:ridge-cosine-lb}) 
  \end{itemize}
  \fussy
Because the parity dataset is an $f$-dataset with a $1/{\sqrt{d}}$-periodic and $\sqrt{d}$-Lipschitz choice of $\phi$, the above results closely match those of Informal Theorem~\ref{ithm:approx}.
We give both results, starting with an upper bound on the minimum-\Rnorm approximate interpolant, which parallels Theorem~\ref{thm:parity-rnorm-approx-ub}.

\begin{restatable}{theorem}{thmperiodicub}\label{thm:periodic-ub}
  Suppose $f \colon \domain \to [-1,1]$ is given by $f(x) = \phi(v^\T x)$ for some unit vector $v \in \sph$ and some $\phi \colon [-\sqrt{d}, \sqrt{d}] \to [-1,1]$ that is $L$-Lipschitz and $\rho$-periodic for $\rho \in [\norm[\infty]{v},1]$.
  Fix any $\epsilon \in (0,1)$.
  There exists a function $g \colon \domain \to \R$ represented by a width-$m$ neural network such that:
  \begin{equation*}
    \cubenorminfty{f - g} \leq \epsilon ;
    \quad
    m \leq dL\polylog(1/\epsilon) \sqrt{\rho\norml[1]{v}} / \epsilon^2 ;
    \quad
    \rnorm{g} \leq L^2\polylog(d/\epsilon) \rho \norml[1]{v} / \epsilon .
  \end{equation*}
\end{restatable}

\begin{remark}\label{rmk:cosine}
    Suppose $f(x) = \cos(\frac{2\pi}{\rho} v^
    \T x)$ for $v \in \flipd^d$ and $\rho \in [\frac1{\sqrt{d}}, 1]$.
    Theorem~\ref{thm:periodic-ub} implies that there exists an $\epsilon$-approximate interpolating neural network $g$ of width $\smash{\tilde{O}(\frac{d^{5/4}}{\sqrt\rho \epsilon^2})}$ and $\smash{\rnorm{g} = \tilde{O}(\frac{\sqrt{d}}{\rho \epsilon})}$.
    If $d$ is even and $\rho = 4/ \sqrt{d}$, then $f(x) = \parity(x)$ for $x \in \flip^d$, and the width and \Rnorm bounds of Theorem~\ref{thm:parity-rnorm-approx-ub} for small $t$ are approximately recovered.
\end{remark}
A detailed version of Theorem~\ref{thm:periodic-ub}
appears in Appendix~\ref{assec:periodic-rnorm-ub}.
The construction is more delicate than that in Theorem~\ref{thm:parity-rnorm-approx-ub} due to the potential lack of symmetries that had existed in the parity dataset.

We give the lower bound on the \Rnorm of all approximately interpolanting ridge functions, whose proof in Appendix~\ref{assec:periodic-rnorm-ridge-approx-lb} relies a reduction to the argument of Theorem~\ref{thm:ridge-parity-lb}.

\begin{restatable}{theorem}{thmridgecosinelb}\label{thm:ridge-cosine-lb}
Assume $d$ is even.
Let $\ridge_d$ be the set of functions $g \colon \domain \to \R$ such that $g(x) = \phi(w^\T x)$  for some $w \in \sph$ and Lipschitz continuous $\phi \colon [-\sqrt{d},\sqrt{d}] \to \R$.
Let $\rho := 4q/\sqrt{d}$ for $q \in \setl{1, 2, \dots, \floorl{\sqrt{d}/4}}$ and $f(x) := \cos((2\pi/(\rho\sqrt{d}))\vec1^\T x)$.
Then
\[
  \inf\setl{\rnorm{g}: g \in \ridge_d, \  \cubenorminfty{g - f} \leq 1/2}
  = \Omega(\sqrt{d}/\rho^2) .
 \]
\end{restatable}

\begin{remark}
    By contrasting the above result to the $\tilde{O}(\frac{\sqrt{d}}{\rho \epsilon})$ \Rnorm of the averaging-based construction from Remark~\ref{rmk:cosine}, ridge functions are suboptimal solutions to \ref{approximate-r-norm} for constant $\epsilon$.
\end{remark}

\begin{remark}\label{rmk:periodic-ridge}
    Lemma~\ref{lemma:lipschitz-approx} (in Appendix~\ref{assec:periodic-rnorm-ub}) implies the existence of a neural network $g_{\ridge} \in \ridge_d$ that point-wise approximates $f$ (i.e., $\cubenorminfty{g_{\ridge} - f} \leq \epsilon$) and has $\smash{\rnorml{g_{\ridge}} = {O}(\frac{\sqrt{d}}{\rho^2 \epsilon})}$.
    Hence, the lower bound in Theorem~\ref{thm:ridge-cosine-lb} is tight when $\epsilon$ is constant.
\end{remark} \section{Conclusion and future work}\label{sec:conclusion}

In this work, we shed light on the \Rnorm inductive bias for learning neural networks, but numerous questions remain.
We are particularly interested in understanding the solutions to \eqref{interpolating-r-norm}
for other datasets, as well as the generality of the averaging techniques used in our constructions.
Extensions of the \Rnorm to deeper networks and analyzing solutions to \eqref{interpolating-r-norm} for other high dimensional datasets could also be useful for proving depth-separation results that focus on variational norm, complementing existing works that focus on width~\citep{tel16,es15,mcpz13,dan17,ss17,ses19}.
Finally, our work suggests that minimizing \Rnorm yields neural networks that are intrinsically high-dimensional, and we are interested in whether this phenomenon is borne out in architectures beyond two-layer fully-connected networks.

\subsubsection*{Acknowledgements}

This work was supported in part by NSF grants CCF-1740833 and IIS-1563785, a JP Morgan Faculty Award, and an NSF Graduate Research Fellowship.

\bibliographystyle{plainnat}
\bibliography{bib}

\appendix

\section{Additional preliminaries}
\label{asec:prelim}

\subsection{Additional definitions and notations}
\label{assec:notation}
We say that $g: \domain \to \R$ is \textit{$k$-index} if there exists a matrix $U \in \R^{k \times d}$ and $\phi: \R^k \to \R$ such that $g(x) = \phi(U x)$ for all $x \in \domain$.
A ridge function is 1-index.

For a matrix $M \in \R^{m \times n}$, we denote the $i$-th largest singular value of $M$ by $\sigma_i(M)$ for $i=1,\dotsc,\min\{m,n\}$.

A random variable $\bu$ is $c$-\emph{subgaussian} if $\norm[\psi_2]{\bu} := \inf\setl{ t\geq0 : \EEl{\exp\parenl{\bu^2/t^2}} \le 2 } \leq c$, and a random vector $\bv$ is $\sigma^2$-\emph{subgaussian} if every one-dimensional projection of $\bv$ is $c$-subgaussian.

The \emph{bias-corrected network} $\bar{g}_{\mu}$ obtained from the (infinite-width) neural network $g_{\mu}$ is given by $\bar{g}_{\mu}(x) := g_{\mu}(x) - g_{\mu}(0)$;
equivalently, $\bar{g}_{\mu}(x) = \int{\parenl{\relu(w^\T x + b) - \relu(b)} \, \mu(\dd w,\dd b)}$. 

The asymptotics implied in the Landau notation (big-$O$, big-$\Omega$, etc.) regard all quantities as potentially increasing functions (e.g., $t$) or decreasing functions (e.g., $\epsilon$, $\delta$, $\alpha$, $\rho$) of the dimension $d$.
The soft-$O$ notation $\tilde{O}(\cdot)$ (only used informally) suppresses terms that are poly-logarithmic in those that appear.
Some of our theorems and lemmas contain an ``if clause'' that uses Landau notation, such as ``if $n \geq O(d^2)$, [\ldots]''.
The interpretation of such a clause is: ``there exists $n_0(d) \in O(d^2)$ such that if $n \geq n_0(d)$, [\ldots]''.
(And, of course, an analogous interpretation should be used when ``$O(d^2)$'' is replaced by other expressions using Landau notation.)

\subsection{Concentration inequalities}
Our proofs make extensive use of textbook probability concentration inequalities.
We provide those results below.

\begin{lemma}[Hoeffding's inequality; Theorem~2.8 in \citealp{blm13}]\label{lemma:hoeffding}
  Let $\bu_1, \dots, \bu_n$ be independent, mean-zero random variables such that $\bu_i$ takes value in $[a_i, b_i]$ almost surely for all $i \in [n]$.
  Then, for any $t > 0$,
  \[\pr{\sum_{i=1}^n \bu_i \geq t} \leq \exp\paren{- \frac{2 t^2}{\sum_{i=1}^n (b_i - a_i)^2}}.\]
\end{lemma}

\begin{lemma}[Multiplicative Chernoff bound; Theorem~4.4 in \citealp{mu17}]\label{lemma:chernoff}
  Let $\bu_1, \dots, \bu_n$ be independent Bernoulli random variables with $\pr{\bu_i = 1} = p \in [0,1]$ for all $i \in [n]$.
  Then, for any $\eta \in (0, 1]$, 
  \[\pr{\sum_{i=1}^n \bu_i \geq (1 + \eta)p} \leq \exp\paren{-\frac{p\eta^2}{3}}.\]
\end{lemma}

\begin{lemma}[Bernstein's inequality; Corollary~2.11 in \citealp{blm13}]\label{lemma:bernstein}
  Let $\bu_1, \dots, \bu_n$ be independent, mean-zero random variables with $\bu_i \leq K$ almost surely for all $i \in [n]$, and let $v \coloneqq \sum_{i=1}^n \EEl{\bu_i^2}$. 
  Then, for any $t > 0$,
  \[\pr{\sum_{i=1}^n \bu_i \geq t} \leq \exp\paren{- \frac{t^2}{2(v + K t/3)}}.\]
\end{lemma}

\begin{lemma}[McDiarmid's inequality; Theorem~6.2 in \citealp{blm13}]\label{lemma:mcdiarmid}
  Let $\bu_1, \dots, \bu_n$ be independent random variables, and let $f$ be a measurable function.
  Suppose, for each $i \in [n]$, the value of $f(u_1, \dots, u_n)$ can change by at most $c_i \geq 0$ by changing the value of $u_i$. Then, for any $t > 0$,
  \[\pr{f(\bu_1, \dots, \bu_n) - \EE{f(\bu_1, \dots, \bu_n)} \geq t} \leq \exp\paren{- \frac{2t^2}{\sum_{i=1}^n c_i^2}}.\]
\end{lemma}

\begin{lemma}[Properties of subgaussian random variables]\label{lemma:subg}
  Let $\bu_1, \dots, \bu_n$ be independent random variables with $\norm[\psi_2]{\bu_i} < \infty$ for all $i \in [n]$.
  There is a universal constant $C>0$ such that the following hold.
  \begin{enumerate}[label=(\roman*)]
      \item (Concentration; Section~2.5.2 in \citealp{vershynin2018high}) For any $t > 0$, $\pr{\abs{\bu_1} \geq t} \leq 2 \exp(-t^2 / (C \norm[\psi_2]{\bu_1}))$. 
      \item (Maximum; Exercise~2.5.10 in \citealp{vershynin2018high}) $\EE{\max_{i \in [n]} \abs{\bu_i}} \leq C \sqrt{\ln n} \max_{i \in [n]} \norml[\psi_2]{\bu_i}$.
      \item (Averaging: Proposition~2.6.1 in \citealp{vershynin2018high}) $\norml[\psi_2]{\sum_{i=1}^n \bu_i}^2 \leq C \sum_{i=1}^n \norml[\psi_2]{\bu_i}^2$.
      \item (Centering; Lemma~2.6.8 in \citealp{vershynin2018high}) $\norm[\psi_2]{\bu_1 - \EE{\bu_1}} \leq C \norm[\psi_2]{\bu_1}$.
      \item (Lipschitzness)
      For any $1$-Lipschitz function $\phi \colon \R \to \R$ with $\phi(0) = 0$, $\norm[\psi_2]{\phi(\bu_1)} \leq \norm[\psi_2]{\bu_1}$.
  \end{enumerate}
\end{lemma}
\begin{proof}
  The only property not already proved in \citep{vershynin2018high} is (v).
  Since $\phi$ is $1$-Lipschitz and $\phi(0) = 0$,
  \[
    \absl{\phi(\bu_1)} = \absl{\phi(\bu_1) - \phi(0)} \leq \absl{\bu_1 - 0} = \absl{\bu_1} .
  \]
  Hence
  \[
    \EE{\exp(\phi(\bu_1)^2/\norml[\psi_2]{\bu_1}^2)} \leq  \EE{\exp(\bu_1^2/\norml[\psi_2]{\bu_1}^2)} \leq 2 ,
  \]
  which implies $\norml[\psi_2]{\phi(\bu_1)} \leq \norml[\psi_2]{\bu_1}$.
\end{proof}

\begin{lemma}[Singular values of random matrices; Theorem~4.6.1 in \citealp{vershynin2018high}]
  \label{lemma:random-matrix-singular-values}
  There is a positive constant $C>0$ such that the following holds.
  Let $\bA$ be a $m \times n$ random matrix whose rows are independent, mean-zero, $v$-subgaussian random vectors in $\R^n$.
  For any $t\geq0$, with probability at least $1-2e^{-t}$, the singular values $\sigma_1(\bA), \sigma_2(\bA), \dotsc, \sigma_n(\bA)$ of $\bA$ satisfy
  \begin{equation*}
    \sqrt{m} - C v (\sqrt{n} + \sqrt{t}) \leq \sigma_i(\bA) \leq \sqrt{m} + Cv (\sqrt{n} + \sqrt{t}) \quad \text{for all $i \in [n]$} .
  \end{equation*}
\end{lemma}

\begin{lemma}[Moment generating function for $\unif(\flip)$;  Lemma~2.2 in \citealp{blm13}]
  \label{lemma:mgf-flip}
  If $\bu \sim \unif(\flip)$, then $\EEl{ \exp(t \bu) } \leq \exp(t^2/2)$ for all $t \in \R$.
\end{lemma}

\subsection{Covering numbers}
Let $\mathcal{N}(\epsilon, A, \gamma)$ denote the \textit{covering number} over a metric space $A$ with metric $\gamma:A\times A \to \R_{+}$.
That is, $\mathcal{N}(\epsilon, A, d) = \min \absl{\mathcal{N}_\epsilon}$ for $\mathcal{N}_\epsilon \subset A$ such that for all $x \in A$, there exists $x' \in \mathcal{N}_\epsilon$ with $\gamma(x, x') \leq \epsilon$.
Note that $\mathcal{N}(\epsilon, [-1,1], \abs{\cdot}) \leq \frac{2}\epsilon$.

\begin{lemma}[Covering numbers of $\sph$; Corollary~4.2.13 in \citealp{vershynin2018high}]\label{lemma:covering}
  For any $\epsilon > 0$, $\mathcal{N}(\epsilon, \sph, \norm[2]{\cdot}) \leq (\frac2\epsilon + 1)^d$.
  If $\epsilon \in [0,1]$, then $\mathcal{N}(\epsilon, \sph, \norm[2]{\cdot}) \leq (\frac3\epsilon)^d$.
\end{lemma}

 \section{Additional properties of \Rnorm}\label{asec:rnorm}

\subsection{Existence and sparsity of the solution to \ref{interpolating-r-norm}}

\begin{proposition}[Lemma 2 in \citealp{pn21b}]
\label{prop:integral-representation}
For any $g: \domain \to \R$ with $\rnorm{g} < \infty$, there exists an even Radon measure\footnote{Evenness of $\mu$ should interpreted in the distributional sense, but it roughly means $\mu(w,b) = \mu(-w,-b)$ when $\mu$ has a density.} $\mu$ over $\sph \times [-\sqrt{d},\sqrt{d}]$, and $v \in \R^d, c\in \R$, such that $g$ admits an integral of the form
$$ 
g(x) = \int_{\sph \times [-\sqrt{d} \times \sqrt{d}]}{ \relu(w^\T x +b) \, \mu(\dd w,\dd b) } + v^\T x + c \quad \forall x \in \domain.
$$
Moreover, $\mu$ attains the \eqref{r-norm}, i.e., $\rnorm{g} = \abs{\mu}$.
\end{proposition}

The following theorem of \citet{pn21b} formalizes the fact that the \Rnorm-minimizing interpolant of a $d$-dimensional dataset can be represented as a finite-width neural network.
\begin{theorem}\label{thm:sparse-vp-solution}
    For any dataset $\{ (x_i,y_i)\}_{i \in [n]}$ from $\domain \times \R$,
    the infimum in \eqref{interpolating-r-norm} is achieved by the sum of an affine function
    $x \mapsto v^\T x + c$ and
    a finite-width neural network $g$ of the form
    $$g(x) = \sum_{j=1}^{m}{a^{(j)} \relu(w^{(j)} x + b^{(j)})} + v^\T x + c,$$
    with $m \le \max\{0, n - (d+1)\}$ and $(w_j,b_j) \in \sph \times [-\sqrt{d},\sqrt{d}]$ for all $i \in [m]$. 
\end{theorem}
\begin{proof}
By Theorem~5 of \citet{pn21b} (see also the proof of Theorem~1 of \citet{pn21a} which covers the interpolation form of the optimization problem), there exists a neural network $x \mapsto \sum_{j=1}^{m'} a^{(j)} \relu(w^{(j)\T} x + b^{(j)})$ of width $m' \le n-(d+1)$, and an affine function $x \mapsto v^{(0)\T} x + c^{(0)}$, such that their sum achieves the infimum in \eqref{interpolating-r-norm}.
We can divide neurons of the neural network into two sets based on whether their corresponding bias term is smaller or larger than $\sqrt{d}$ in absolute value.
Since every $x \in \domain$ satisfies $\norml[2]{x} \leq \sqrt{d}$,
without loss of generality (by possibly flipping the sign of some $a^{(j)}$ and $w^{(j)}$), assume the first $m$ neurons satisfy $\abs{b^{(j)}} \le \sqrt{d}$ and the rest satisfy $b^{(j)} > \sqrt{d}$. Then we have
\begin{align*}
g(x) &= \sum_{j=1}^{m}{a^{(j)} \relu(w^{(j)\T} x + b^{(j)})} + \sum_{j=m+1}^{m'}{a^{(j)} \relu(w^{(j)\T} x + b^{(j)})} + v^{(0)\T} x + c^{(0)} \\
&= \sum_{j=1}^{m}{a^{(j)} \relu(w^{(j)\T} x + b^{(j)})} + \sum_{j=m+1}^{m'}{a^{(j)} (w^{(j)\T} x + b^{(j)})} + v^{(0)\T} x + c^{(0)} \\
&= \sum_{j=1}^{m}{a^{(j)} \relu(w^{(j)\T} x + b^{(j)})} + \Bigl( \underbrace{v^{(0)}  + \sum_{j=m+1}^{m'}{a^{(j)}w^{(j)}}}_{=: v} \Bigr)^\T x + \underbrace{\sum_{j=m+1}^{m'}{b^{(j)}} + c^{(0)}}_{=: c}.
\end{align*}
Therefore, $g$ has the desired form with $m \leq m' \leq n - (d+1)$.
\end{proof}
 \section{Extending our results to a different variational norm}\label{asec:vnorm}

\newcommand{\Vnorm}{$\mathcal{V}_2$-norm\xspace}

This paper considers the approximation and generalization implications of bounding the complexity of shallow neural networks with the \Rnorm.
However, \Rnorm is not the only weight-based complexity measurement, and other works employ slightly different norms for similar purposes.
This appendix demonstrates that our results are not peculiarities of our formulation of \Rnorm and extend to other variational norms.
One alternative---which we refer to as the \textit{\Vnorm}---omits the linear component of the neural network whose measure determines the \Rnorm  and instead permits ReLU neurons whose thresholds lie outside the domain $\domain$.

We first introduce notation for an infinite-width neural network that permits such thresholds.
Let $\vmeasures$ denote the space of probability measures over $\sph \times [-2\sqrt{d},2\sqrt{d}]$. 
For some measure $\tilde{\mu} \in \vmeasures$, let $\tilde{g}_\mu: \domain \to \R$ be an infinite-width neural network with
\[\tilde{g}_{\tilde\mu}(x) = \int_{\sph \times [-2\sqrt{d},2\sqrt{d}]} \relu(w^\T x + b) \, \tilde\mu(\dd w, \dd b),\]
which has total variation norm $\abs{\tilde\mu} = \int_{\sph \times [-2\sqrt{d},2\sqrt{d}]} \abs{\tilde\mu}(\dd w, \dd b)$.
Now, we introduce the \Vnorm for some $g: \domain \to \R$:
\begin{equation}
    \label{v-norm}
    \tag{\Vnorm}
    \vnorm{g} = \inf_{\tilde\mu \in \vmeasures}{\abs{\tilde\mu}} \quad \text{s.t.} \quad g(x) = \tilde{g}_{\tilde\mu}(x), \quad \forall x \in \domain.
\end{equation}
In the same spirit as Lemma \ref{lemma:rnorm-relu-l1}, for a discrete network $g(x) = g_{\theta}(x)$ with $\theta = (a^{(j)}, w^{(j)}, b^{(j)})_{j \in [m]} \in (\R \times \sph \times [-2\sqrt{d}, 2\sqrt{d}])^m$, we have $\vnorm{g} \leq \norm[1]{a}$.

Our definition of the \Vnorm was introduced by \citet{sx21} as the norm corresponding to their variation space $\mathbb{P}_1$ with constants $c_1 = -2\sqrt{d}$ and $c_2 = 2\sqrt{d}$.
They relate the \Vnorm to the \emph{Barron norm} of \citet{emw19} and the Radon norm of \citet{owss19}.
We show that the \Vnorm and the \Rnorm are closely related and that all of our bounds apply equivalently to the \Vnorm.
We first place upper and lower bounds on the $\vnorm{g}$ in terms of $\rnorm{g}$ and then explain why each of our results transfers to this new variational norm.

\begin{theorem}\label{thm:vr-norm}
    Suppose $g: \domain \to \R$ has $\rnorm{g} < \infty$. Then, $\rnorm{g} \leq \vnorm{g}$.
If $g$ is bounded near the origin (i.e., $\abs{g(x)} \leq K$ for all $x$ with $\norm[2]{x} \leq 1$), then $\vnorm{g} \leq 12\rnorm{g} + 18K$.
\end{theorem}

As a result, all of our results that apply to the \Rnorm translate modulo constants to the \Vnorm. 
Because $\vnorm{g} \geq \rnorm{g}$ always holds, every theorem that places lower bounds on an \Rnorm exactly translates to $\vnorm{g}$, including Theorems~\ref{thm:ridge-parity-lb}, \ref{thm:parity-rnorm-lb}, \ref{thm:parity-sample-rnorm-lb}, and \ref{thm:parity-gen-ub}.
The upper-bounds hold up to constants by observing that every target function we consider is bounded by some $K$ on $\domain$.

\begin{itemize}
    \item Because every sawtooth $s_{w,t}$ is bounded by 1, the averages of sawtooths $g$ in Theorems~\ref{thm:parity-rnorm-ub} and \ref{thm:parity-rnorm-approx-ub} are bounded by $K = \frac1q = O(\sqrt{d})$. Hence, $\vnorm{g} = O(d)$, just like $\rnorm{g}$.
    \item For the ``cap'' construction $\bg$ of Theorem~\ref{thm:parity-gen-lb}, there are $k = O(n \log (d) / d)$ neurons, none of which are active at the origin. Their biases are negative and---under the ``good event''---their weight norms are $O(1 / \log d)$. Thus, no neuron can output a value greater than $O(1 / \log d)$, so even if all $k$ neurons activate, every $x$ with $\norm[2]{x} \leq 1$ has $\abs{\bg(x)} = O(n / d)$, which is dominated by the \Rnorm of $O(n \sqrt{\log d} / d)$.
    \item The construction $g$ of Theorem~\ref{thm:periodic-ub} computes an average of functions bounded on $[-1, 1]$. 
    Therefore, $\bg$ is bounded by 1, and its \Vnorm is no more than its \Rnorm.
\end{itemize}

\begin{proof}[Proof of Theorem~\ref{thm:vr-norm}]
We show separately that $\vnorm{g} \geq \rnorm{g}$, and then
that $\vnorm{g} \leq 12 \rnorm{g} + 18K$ under the additional hypothesis that $|g(x)| \leq K$ for all $x \in \domain$ such that $\norml[2]{x} \leq 1$.

\textbf{Lower bound on \Vnorm:}
Fix any $\xi > 0$.
By the definition of \Vnorm, there exists $\tilde\mu \in \vmeasures$ such that $g(x) = \tilde{g}_{\tilde\mu}(x)$ for all $x \in \domain$ and $\absl{\tilde\mu} \leq \vnorm{g} + \xi$.\footnote{This relies on the assumption that $\vnorm{g} < \infty$, but if it is not, then the claim trivially follows because $\rnorm{g} < \infty$.}
We show that there exists $g_\mu$ (where $\mu$ is $\tilde{\mu}$ with the support of $b$ restricted to $[-\sqrt{d}, \sqrt{d}]$), $v$, and $c$ such that $\tilde{g}_{\tilde\mu}(x) = g_\mu(x) + v^\T x + c$ for all $x \in \domain$. Observe that for any $x \in \domain$, $w^\T x + b > 0$ if $b > \sqrt{d}$ and $w^\T x + b < 0$ if $ b < -\sqrt{d}$.
\begin{align*}
    \tilde{g}_{\tilde\mu}(x)
    &=  \int_{\sph \times [-2\sqrt{d},-\sqrt{d}]} \relu(w^\T x + b)\tilde\mu(\dd w, \dd b) + \int_{\sph \times [-\sqrt{d},\sqrt{d}]} \relu(w^\T x + b)\tilde\mu(\dd w, \dd b)\\ &\quad + \int_{\sph \times [\sqrt{d},2\sqrt{d}]} \relu(w^\T x + b) \tilde\mu(\dd w, \dd b) \\
    &= 0 + \int_{\sph \times [-\sqrt{d},\sqrt{d}]} \relu(w^\T x + b)\mu(\dd w, \dd b) + \int_{\sph \times [\sqrt{d},2\sqrt{d}]} (w^\T x + b) \tilde\mu(\dd w, \dd b) \\
    &= g_\mu(x) + \sum_{i=1}^d x_i \underbrace{\int_{\sph \times [\sqrt{d},2\sqrt{d}]} w_i \tilde\mu(\dd w, \dd b)}_{\coloneqq v_i} + \underbrace{\int_{\sph \times [\sqrt{d},2\sqrt{d}]} b \tilde\mu(\dd w, \dd b)}_{\coloneqq c} \\
    &=  g_\mu(x) + v^\T x + c.
\end{align*}

As a result, $\rnorm{g} \leq \abs{\mu} \leq \absl{\tilde\mu} \leq \vnorm{g}  + \xi$.
Because the argument holds simultaneously for all $\xi > 0$, we conclude that $\rnorm{g} \leq \vnorm{g}$.

\textbf{Upper bound on \Vnorm:}
By Proposition~\ref{prop:integral-representation}, there exist $\mu \in \vmeasures$, $v \in \R^d$, and $c \in \R$ such that $g(x) = g_\mu(x) + v^\T x + c$ for all $x \in \domain$ and $\abs\mu = \rnorm{g}$.
We construct $\tilde\mu \in \vmeasures$ such that $g_\mu(x) + v^\T x + c = \tilde{g}_{\tilde\mu}(x)$ for all $x \in \domain$:\footnote{In the event that $v = \vec0$, we use $\frac{v}{\norm[2]{v}} := \frac1{\sqrt{d}} \vec1 \in \sph$.}
\[\tilde{\mu}(w, b) = \begin{cases}
    \mu(w, b) & \text{if $b \in [-\sqrt{d}, \sqrt{d}]$} ; \\
    \paren{-3\norm[2]{v} + \frac{2c}{\sqrt{d}}}\delta\paren{\paren{w, b} - \paren{v,2\sqrt{d}}} \\\quad + \paren{4\norm[2]{v} - \frac{2c}{\sqrt{d}}} \delta\paren{\paren{w, b} - \paren{v,\frac32\sqrt{d}}} & \text{otherwise}.
    \end{cases}
\]
    
Fix any $x \in \domain$. Then:
\begin{align*}
    \tilde{g}_{\tilde\mu}(x) - g_\mu(x)
    &=  \paren{-3\norm[2]{v} + \frac{2c}{\sqrt{d}}} \relu\paren{\frac{v^\T}{\norm[2]{v}}  x +  2\sqrt{d}} \\&\quad + \paren{4\norm[2]{v} -  \frac{2c}{\sqrt{d}}}\relu\paren{\frac{v^\T}{\norm[2]{v}} x +  \frac32\sqrt{d}} \\
    &= \paren{-3\norm[2]{v} + \frac{2c}{\sqrt{d}}} \paren{\frac{v^\T}{\norm[2]{v}}  x +  2\sqrt{d}} + \paren{4\norm[2]{v} -  \frac{2c}{\sqrt{d}}}\paren{\frac{v^\T}{\norm[2]{v}} x +  \frac32\sqrt{d}} \\
    &= v^\T x + c.
\end{align*}
Therefore, $\absl{\tilde\mu} \leq \abs\mu + \absl{-3\norm[2]{v} + \frac{2c}{\sqrt{d}}} + \absl{4\norm[2]{v} -  \frac{2c}{\sqrt{d}}} \leq \abs\mu +  7 \norml[2]{v} + \frac{4\abs{c}}{\sqrt{d}}$.
It suffices to bound $\norm[2]{v}$ and $\abs{c}$.
\begin{itemize}
\item Let $x_0\coloneqq \frac{v}{\norm[2]{v}}$. By boundedness, the triangle inequality, and several applications of Holder's inequality:
    \begin{align*}
        \abs{g(x_0) - g(0)} &\geq \abs{v^\T x_0 } - \abs{g_\mu(x_0) - g_\mu(0)} \\
        &= \norm[2]{v} - \abs{\int_{\sph \times [-\sqrt{d}, \sqrt{d}]} (\relu(w^\T x_0 + b) - \relu( b))\mu(\dd w, \dd b)} \\
        &\geq \norm[2]{v} - \int_{\sph \times [-\sqrt{d}, \sqrt{d}]} \abs{\relu(w^\T x_0 + b) - \relu(b)}\abs{\mu}(\dd w, \dd b) \\
        &\geq \norm[2]{v} - \norml[2]{x_0} \abs{\mu}
        =\norm[2]{v} - \abs{\mu}.
    \end{align*}
    Hence, $\norm[2]{v} \leq \abs{g(x_0) - g(0)} + \abs{\mu} \leq 2K + \abs{\mu}$.

    \item We similarly employ our bound on $g(0)$:
\begin{align*}
        K
        &\geq \abs{g(0)} 
        \geq \abs{c} - \abs{\int_{\sph \times [-\sqrt{d}, \sqrt{d}]} \relu(b) \mu(\dd w, \dd b)} 
        \geq \abs{c} - \abs{\mu} \sqrt{d}.
    \end{align*}
    As a result, $\abs{c} \leq K + \abs\mu \sqrt{d}$.
\end{itemize}
Therefore, $\vnorm{g} \leq \absl{\tilde\mu} \leq 12\abs{\mu} + 18K \leq 12 \rnorm{g} + 18K$.
\end{proof}

 \section{Proofs for Section~\ref{sec:parity-approx}}\label{asec:parity-approx}

\subsection{Proofs for Section~\ref{ssec:parity-rnorm-ridge-approx-lb}}\label{assec:parity-rnorm-ridge-approx-lb}

The proof of Theorem~\ref{thm:ridge-parity-lb} relies on the following lemma, which is essentially a robust version of the mean value theorem.

\begin{lemma}\label{lemma:calc}
    Suppose $\phi: \R \to \R$ is Lipschitz continuous on the interval $[t_1, t_2]$ with $\phi(t_1) \neq \phi(t_2)$.
    Define
    \begin{multline*}
    A := \biggl\{ t \in [t_1, t_2] :
    \text{$\phi'(t)$ exists} , \\
    \abs{\phi'(t)} \geq \frac12 \cdot \frac{\abs{\phi(t_2) - \phi(t_1)}}{t_2 - t_1},
    \ \sign{\phi'(t)} = \sign{\phi(t_2) - \phi(t_1)} \biggr\}.
    \end{multline*}
    (The factor $1/2$ in the definition of $A$ can be replaced by any constant in $(0,1)$.)
    Then, $\Leb(A) > 0$, where $\Leb$ is the Lebesgue measure.
\end{lemma}
\begin{proof}
Recall that $\phi'$ denotes the right continuous derivative of $g$ (or the right-hand Dini derivative) which is guaranteed to exist except on a null set by Rademacher's theorem.
Let $s := \sign{\phi(t_2) - \phi(t_1)}$.
By the assumption $\phi(t_1) \neq \phi(t_2)$ and the Fundamental Theorem of Calculus, we have
\[
0 < \absl{\phi(t_2) - \phi(t_1)}
= s\paren{\phi(t_2) - \phi(t_1)}
= \int_{t_2}^{t_1} s \phi'(z) \dd z 
\leq (t_2 - t_1) \esssup_{z \in [t_1, t_2]} s \phi'(z).
\]
Recall that, by definition,
\[
  \esssup_{z \in [t_1, t_2]} s \phi'(z)
  = \inf\left\{
    a :
    \Leb(\{
        z \in [t_1,t_2] :
        \text{$\phi'(z)$ exists} , \
        s \phi'(z) > a
    \}) = 0
  \right\} ,
\]
and thus
\[
  B :=
  \left\{
    z \in [t_1,t_2] :
    \text{$\phi'(z)$ exists} , \
    s(\phi(t_2) - \phi(t_1)) \leq 2 \cdot (t_2 - t_1) s \phi'(z)
  \right\}
\]
satisfies $\Leb(B) > 0$.
Observe that $B=A$, so $\Leb(A) > 0$, concluding the proof.
\end{proof}

\subsection{Proofs for Section~\ref{ssec:parity-rnorm-ub}}\label{assec:parity-rnorm-ub}

\begin{restatable}[Detailed version of Theorem~\ref{thm:parity-rnorm-approx-ub}]{theorem}{thmparityrnormapproxubfull}
  \label{thm:parity-rnorm-approx-ub-full}
There exists a universal constant $C$ such that for any even $d$, even sawtooth width $t \in \set{0, 2, \dots d}$, and accuracy $\epsilon \in (0, \frac12)$, there exists a $k$-index (see Appendix~\ref{assec:notation}) width-$m$  neural network $g$ with  
$\cubenorminfty{g - \parity} \leq \epsilon$ such that:
\begin{enumerate}
    \item If $t \leq C \sqrt{d \ln \frac1\epsilon}$, then $k = O(\frac{d^{3/2}}{t+1} \cdot \frac{\log^{1/2}\epsilon}{\epsilon^2})$, $m = O(d^{3/2} \frac{\log^{1/2}\epsilon}{\epsilon^2})$, and $\rnorm{g} = O(d \log \frac1\epsilon)$.
    \item Otherwise, $k=O(\frac{d^2}{\epsilon t^2})$, $m = O(\frac{d^2}{\epsilon t})$, and $\rnorm{g} = O(t \sqrt{d})$.
\end{enumerate}
\end{restatable}

\begin{proof}
For  $\bw^{(1)}, \dots, \bw^{(k)} \simiid \unif(\flip^d)$ and $q := \pr{\abs{\bw^{(1)\T} x} \leq t}$ (for any $x \in \flip^d$), let \[\bg(x) := \frac{1}{kq} \sum_{j=1}^k s_{\bw^{(j)}, t}(x).\]
Because $\EE{s_{\bw, t}(x)} = q \cdot  \parity(x)$, we have $\EE{\bg(x)} = \parity(x)$.
Following the arguments in the proof of Theorem~\ref{thm:parity-rnorm-ub}, we have
$\rnorm{\bg} = O(\frac{(t+1) \sqrt{d}}{q})$, and $\bg$ has width $O(k(t+1))$.

It remains to place a lower bound on $q$ and to show that with non-zero probability, $\bg$ uniformly approximates $\parity$.
By applying a union bound, it suffices to show that $\pr{\abs{\bg(x) - \parity(x)} \geq \epsilon} \leq \frac{1}{2^{d+1}}$ for any fixed $x \in \flip^d$.

For fixed $x$, let $\br(x) := \absl{\setl{j \in [k]: \absl{x^\T \bw^{(j)}} > t}}$.
We upper-bound the accuracy of approximation of $\parity(x)$ by $\bg(x)$ in terms of $\br(x)$:
\begin{multline*}
    \abs{\bg(x) - \parity(x)}
    = \abs{\frac1{qk} \sum_{j=1}^k \indicatorl{\absl{x^\T \bw^{(j)}} \leq t}-q} \\
    = \abs{\frac{(k - \br(x))(1 - q)}{qk}-\frac{\br(x)q}{qk}}
    = \abs{\frac{1 - q}{q} - \frac{\br(x)}{qk}}.
\end{multline*}
Define \smash{$T := 2\ceill{\sqrt{(d/2) \ln(8/\epsilon)}}$}, and note that $\pr{|\bw^\T x| \geq T} \leq \frac{\epsilon}4$.
The proof naturally divides into two cases, depending on the value of $t$.

\textbf{Case 1: $t \leq T$.}
We first lower-bound $q$.
Because $\bw^\T x$ is a shifted symmetric binomial distribution around $\bw^\T x = 0$, if $\abs{t'} \geq \abs{t}$ and $t' \equiv t \pmod 2$, then $\pr{\bw^\T x = t'} \leq \pr{\bw^\T x = t}$. Then, for any $t \leq T$:
\begin{align*}
q
&= \sum_{\tau=-t/2}^{t/2} \pr{\bw^\T x = 2\tau} 
= (t+1) \cdot \frac1{t+1} \sum_{\tau=-t/2}^{t/2} \pr{\bw^\T x = 2\tau} \\
&\geq (t+1) \cdot \frac1{T+1} \sum_{\tau=-T/2}^{T/2} \pr{\bw^\T x = 2\tau}
= \frac{t+1}{T+1} \pr{|\bw^\T x| \leq T} \\
&\geq \frac{(t+1)(1 - \frac\epsilon2)}{2\sqrt{d \ln \frac4\epsilon }} 
\geq \frac{t+1}{4\sqrt{d \ln \frac4\epsilon }} .
\end{align*}

Now, we bound $\br(x)$ by Bernstein’s inequality (Lemma~\ref{lemma:bernstein}) by taking $k \geq \frac{C d^{3/2} \sqrt{\ln \frac1\epsilon}}{\epsilon^2 (t+1)}$:
\begin{align*}
\pr{|\bg(x) - \parity(x)| > \epsilon}
&= \pr{\abs{\br(x) - \EEl{\br(x)}} > \epsilon qk} \\
& \leq 2\exp\left(-\frac{\epsilon^2 q^2 k^2 }{2(kq(1 - q) + \epsilon q k/3) }\right) \\
& \leq 2\exp\left(-\frac{\epsilon^2 k(t+1)}{8(1+\epsilon/3)\sqrt{d \ln \frac4\epsilon}} \right)
\leq \frac1{2^{d+1}}.
\end{align*}

\textbf{Case 2: $t \geq T$.}
By Hoeffding's inequality (Lemma~\ref{lemma:hoeffding}) and the assumption on $t$, we have $q \geq 1 - 2\exp(-\frac{2 t^2}{d}) \geq 1 - \epsilon/4 \geq 3/4$.
Observe that $\EE{\br(x)} = (1-q)k \leq \frac{\epsilon k}{4}$.

We show that $ \abs{\bg(x) - \parity(x)}= \absl{\frac{1 - q}{q} - \frac{\br(x)}{qk}} \leq \epsilon$ by showing that $\br(x) \leq (1 - q)k + \epsilon qk$ and $\br(x) \geq (1 - q)k - \epsilon qk$.
Because $1-q \leq \frac\epsilon4$ and $q \geq \frac34$, $(1 - q)k - \epsilon qk \leq - \frac\epsilon2k$, so the second inequality is always satisfied because $\br(x) \geq 0$.
For the former inequality, it suffices to show that $\br(x) \leq \frac{3\epsilon k}4$ with probability at least $1-2^{-(d+1)}$.
We take $k \geq \frac{Cd^2}{\epsilon t^2}$,
which implies $k \geq C(\frac{d}{2\epsilon} + \frac{de^{-2t^2/d}}{2\epsilon^2}) \geq C(\frac{d}{2\epsilon} + \frac{d(1-q)}{4\epsilon^2})$ by the bounds on $t$ and $q$.
Then, by Bernstein's inequality (Lemma~\ref{lemma:bernstein}),
we have
\begin{align*}
\pr{|\bg(x) - \parity(x)| > \epsilon}
& = \pr{\abs{\br(x) - \EEl{\br(x)}} > \frac{3\epsilon k}4} \\
& \leq 2\exp\paren{-\frac{9\epsilon^2k^2/16}{2(kq(1-q) + \epsilon k/4)}}
\leq \frac1{2^{d+1}} ,
\end{align*}
so the claim follows.
\end{proof}

\subsection{Proofs for Section~\ref{ssec:parity-rnorm-approx-lb}}\label{assec:parity-rnorm-approx-lb}

\thmparityrnormlb*
\begin{proof}
Consider any measure $\mu$ over $\sph \times [-2 \sqrt{d}, 2 \sqrt{d}]$, $v \in \R^d$, and $c \in \R$ such that $g(x) = g_\mu(x) + v^\T x + c = \int_{\sph \times \R} \relu(w^\T x + b) \mu(\dd w, \dd b)$ for all $\norm[2]{x} \leq d$.
We prove the claim by showing that $\abs{\mu} \geq \frac{\alpha d}{8}$ for any such $\mu$.

By Fact~\ref{fact:ip-norm}, $\cubenormtwol{g - \parity} \leq 1 - \alpha$ implies that $\innerprod{g}{\parity} \geq \alpha.$
We show that this inner product bound is only possible if $\abs{\mu}$ is sufficiently large. 
By Lemma~\ref{lemma:relu-parity-correlation}, any fixed neuron $r_{w,b}(x) := \relu(w^\T x + b)$ has $\abs{\innerprod{ r_{w,b}}{\parity}} \leq \frac8d$. 
Because the inner-product over $\flip^d$ is a discrete sum and $\parity$ is orthogonal to any affine function (such as $x \mapsto v^\T x + c$), we can upper-bound the ability of $g$ to correlate with $\parity$ as follows:
\begin{align*}
\cubeinnerprod{g,\parity} 
&= \int_{\sph \times \R} \cubeinnerprod{r_{w,b},\parity} \mu(\dd w, \dd b)  + \EE{(v^\T \bx + c) \parity(\bx)}\\
&\leq \int_{\sph \times \R} \abs{\cubeinnerprod{r_{w,b},\parity}} \abs{\mu(\dd w, \dd b)} \\
&\leq \frac8d \int_{\sph \times \R} \abs{\mu(\dd w, \dd b)} 
= \frac{8 \abs{\mu}}d.
\end{align*}
Thus, $\innerprod{g_\mu}{\parity} \geq \alpha$ only if $\abs{\mu} \geq \frac{\alpha d}{8}$.
\end{proof}

\begin{fact}\label{fact:ip-norm}
    For any measure $\nu_0$ over $\flip^d$, $g \in L^2(\nu_0)$, $h: \flip^d \to \flip$, and $\alpha \in (0, 1)$, if $\normnuzero{g - h} \leq 1 - \alpha$, then $\innerprodnuzero{g}{h} \geq \alpha$.
\end{fact}
\begin{proof}
The claim is a consequence of the fact $\innerprodnuzero{h}{h} = 1$ and Cauchy-Schwarz:
\begin{align*}
\innerprodnuzero{g}{h}  
&= \innerprodnuzero{h}{h} + \innerprodnuzero{g - h}{h} \\
&= 1 + \innerprodnuzero{g - h}{h} \\
&\ge 1 - \normnuzero{g - h} \\
&\ge 1 - (1-\alpha) = \alpha.\qedhere
\end{align*} 
\end{proof}

\begin{lemma}\label{lemma:relu-parity-correlation}
    For $d \geq 8$, $w \in \R^d$ with $\norm[2]{w} \leq 1$, and $b \in [-2\sqrt{d}, 2\sqrt{d}]$, the neuron $r_{w, b}(x) := \relu(w^\T x+b)$ satisfies $\abs{\innerprod{r_{w, b}}{\parity}} \leq \frac8d$.
\end{lemma}
\begin{remark}
    Lemma~\ref{lemma:relu-parity-correlation} is asymptotically tight. 
    For even $d$, consider the ``single-blade'' sawtooth function \[s_{\vec1, 0}(x) = \sqrt{d}(r_{\vec1/\sqrt{d}, 1/\sqrt{d}}(x) - 2r_{\vec1/\sqrt{d}, 0}(x) + r_{\vec1/\sqrt{d}, -1/\sqrt{d}}(x))\] that satisfies $s_{\vec1,0}(x) = \parity(x) \indicatorl{\vec1^\T x = 0}$. 
    Then, \[\cubeinnerprodl{s_{\vec1, 0}, \parity} = \frac1{2^d}{d \choose d/2} \geq \frac1{\sqrt{2d}},\] and thus there exists $b$ with \smash{$\absl{\cubeinnerprodl{r_{\vec1/\sqrt{d}, b}, \parity}} \geq \frac1{4 \sqrt{2} d}$}.
\end{remark}
\begin{proof}
We directly bound the inner product by showing that we can bound a discrete second derivative. 
For any $x \in \flip^d$, let $x^{j} \in\flip^d$ denote $x$ with a flipped $j$th bit.
That is, $x^j_i = (-1)^{\indicator{i=j}} x_i $.
Observe that $\parity(x) = - \parity(x^j)$.
\begin{align*}
\abs{\innerprod{ r_{w,b}}{\parity}}
&= \frac{1}{2^d}\abs{ \sum_{x} r_{w,b}(x) \parity(x)} \\
&= \frac{1}{4 \cdot 2^d}\abs{ \sum_{x} (r_{w,b}(x) \parity(x) + r_{w,b}(x^j)\parity(x^j)+ r_{w,b}(x^{j'})\parity(x^{j'}) + r_{w,b}(x^{j,j'})\parity(x^{j,j'})) } \\
&= \frac1{4 \cdot 2^d} \left|\sum_x \parity(x) (r_{w,b}(x) - r_{w,b}(x^{j}) - r_{w,b}(x^{j'}) + r_{w,b}(x^{j,j'})) \right| \\
&\leq \frac{1}{4 \cdot 2^d} \sum_x |r_{w,b}(x) - r_{w,b}(x^{j}) - r_{w,b}(x^{j'}) + r_{w,b}(x^{j,j'})|.
\end{align*}

We say that $(x, x^j)$ is \textit{cut} and denote $(x, x^j) \in C_j$ if $x$ and $x^j$ lie on the opposite side of the ``hinge'' of the neuron $r_{w,b}$, that is $\text{sign}(w^\T x - b) \neq \text{sign}(w^\T x^j - b)$.
Let $S_{x, j, j'} = \setl{x, x^j, x^{j'}, x^{j, j'}}$ represent a ``square'' in $\flip^d$, and let $S_{x, j, j'} \in C_{j, j'}$ if any of its edges $(x, x^j), (x, x^{j'}), (x^j, x^{j,j'}), (x^{j'}, x^{j, j'})$ are cut.
We bound the term inside the sum by considering two cases.

\begin{enumerate}
    \item 
If $S_{x, j, j'} \not\in C_{j, j'}$,
    then $|r_{w,b}(x) - r_{w,b}(x^{j}) - r_{w,b}(x^{j'}) + r_{w,b}(x^{j,j'})| = 0$.
    \item Otherwise, the quantity is bounded by Lipschitzness:
    \begin{align*}
    |r_{w,b}(x) - r_{w,b}(x^{j}) - r_{w,b}(x^{j'}) + r_{w,b}(x^{j,j'})|
    &\leq |r_{w,b}(x) - r_{w,b}(x^{j})| + |r_{w,b}(x^{j'}) - r_{w,b}(x^{j,j'})| \\
    &\leq |w^\T x - w^\T x^{j}| +|w^\T x^{j'} - w^\T x^{j, j'}|
    = 4|w_{j}|.
    \end{align*}
\end{enumerate}

Therefore, $\absl{\innerprodl{r_{w,b}}{\parity}} \leq \min_{j \neq j'} \frac1{2^d} \abs{C_{j, j'}} \absl{w_j}$.
It remains to bound $\abs{C_{j, j'}}$ and $\absl{w_j}$ for some $j$ and $j'$.
By employing a bound on the total number of cut edges \citep{oneil71}:
\begin{align*}
    \frac{1}d \sum_{j=1}^d \abs{C_j}
    &\leq \frac1{2d} \cdot \ceil{\frac{d}2} {d \choose \floor{d/2}}
    \leq \frac{  2^d}{2\sqrt{d} }.
\end{align*}
As a result, at most $\frac{d}2$ choices of $j$ satisfy $\absl{C_{j}} \geq \fracl{ 2^d}{\sqrt{d}}$.
Because $\norm[2]{w} \leq 1$, at most $\frac{d}{4}$ coordinates $j$ have $|w_j| \geq \fracl{2}{\sqrt{d}}$. 
Thus, there exist at \textit{least} $\frac{d}4$ coordinates $j$ satisfying both $\absl{C_{j}} \leq \fracl{2^d}{\sqrt{d}}$ and $|w_j | \leq \fracl2{\sqrt{d}}$. 
Assuming $d \geq 8$, let $j, j'$ be two of those coordinates.
Since $\absl{C_{j, j'}} \leq 2\absl{C_j} + 2 \absl{C_{j'}}$, we conclude that $\absl{C_{j, j'}} \leq \fracl{4 \cdot 2^d}{\sqrt{d}}$, which gives the desired bound on the inner product.
\end{proof}

\subsection{\Rnorm lower bound for sampled parity datasets}
\label{asec:parity-approx-sampled}

\begin{restatable}{theorem}{thmparitysamplernormlb}\label{thm:parity-sample-rnorm-lb}
Fix any $\delta \in (0,1)$ and $\alpha = \omega(1/d)$, and assume $n \geq O \parenl{d^3(\log d + \log\parenl{1/\delta})}$.
With probability at least $1 - \delta$, $\inf \setl{ \rnorm{g} : \empnorm{g - \parity} \le 1 - \alpha } \geq \Omega(\alpha d)$.
\end{restatable}

\begin{proof}
Let $g$ be a function with finite \Rnorm which satisfies the $L^2(\nun)$ approximability condition, which admits an integral representation due to Proposition~\ref{prop:integral-representation}. That is, 
\[
g(x) = \int_{\sph \times [-\sqrt{d}, \sqrt{d}]}{(\relu(w^\T x + b) -\relu(b)) \mu(\dd w,\dd b)} + c + v^\T x \quad \forall x \in \domain
\]
for some measure $\mu$ and $v \in \R^d, c = g(0)$. Moreover, $g$ can be represented compactly as $g(x) = \bar{g}_{\mu}(x) + v^\T x + c$ where $\bar{g}_{\mu}(x) = g_{\mu}(x) - g_{\mu}(0)$. 

By Fact~\ref{fact:ip-norm}, $\empnorm{g - \parity} \leq 1 - \alpha$ only if $\empinnerprod{g, \parity} \geq \alpha$.
We use this correlation to prove lower bounds on  $\abs{\mu}$ (the total variation of measure $\mu$).
At a high level, we upper-bound \[\empinnerprod{g, \parity} = \empinnerprod{\bar{g}_{\mu}, \parity} + \empinnerprod{v^\T x + c, \parity(x)}\]
in terms of $\abs{\mu}$ by relating quantities in $L^2(\nun)$ with their $L^2(\nu)$ counterparts. 
We show that each component of the sum is small for sufficiently large $n$ and $d$.

We first bound the correlation of the linear combination of neurons with parity, proving upper bounds on $\empinnerprod{g, \parity}$. 
We denote $\bar{r}_{w,b}(x) = \relu(w^\T x + b) -\relu(b)$ to be the adjusted ReLU.
By the triangle inequality,
\begin{align*}
    \empinnerprod{\bar{g}_{\mu}, \parity} &\le  \int_{\sph \times [-\sqrt{d}, \sqrt{d}]}{\abs{\empinnerprod{\bar{r}_{w,b},\parity}} \abs{\mu}(dw,db)} \\
    &\le \abs{\mu} \sup_{\substack{w \in \sph,  b \in [-\sqrt{d}, \sqrt{d}]}}{\abs{ \empinnerprod{\bar{r}_{w,b} , \parity} }}.
\end{align*}

Lemmas~\ref{lemma:relu-parity-correlation} and \ref{lemm:nonlinear-term-bound} together bound the correlation of any neuron $\bar{r}_{w,b}$ with $\parity$.
That is, for any $w \in \sph$ and  $b \in [-\sqrt{d},\sqrt{d}]$, with probability at least $1-\delta/3$:
\[
\abs{ \empinnerprod{\bar{r}_{w,b} , \parity} } \le \abs{ \cubeinnerprod{\bar{r}_{w,b} , \parity} } + C_1 \sqrt{\frac{d \parenl{\ln{n} + \ln\parenl{3/\delta}} }{n}  } \le  \frac{8}{d} + 2C_1 \sqrt{\frac{d \ln{n}}{n}  } \le \frac{C_2}{d},
\]
where $C_1$ is the constant from Lemma \ref{lemm:nonlinear-term-bound} and $n > C \parenl{d^3 (\log{d} + \log\parenl{1/\delta})}$ by assumption. 

We now show that the linear components cannot be substantially correlated with the parity function and bound $\empinnerprod{v^\T x + c, \parity}$.
Because no linear term correlates with the full parity dataset, Lemma~\ref{lemm:linear-term-bound} provides an upper bound on the inner product between the linear perturbation and sampled parity dataset and implies the following bound with probability at least $1-\delta/3$:
\begin{align*}
    \label{first-term}
    \abs{\empinnerprod{ v^\T x + c, \parity}} &\le  8\max\{\abs{\mu}, 1\} \sup_{\substack{\abs{c} \le 1 \\ \norm[2]{v} \le 1}}{ \abs{ \empinnerprod{v^{\T} x + c, \parity} }} \\
    &= 8\max\{\abs{\mu}, 1\}\left( \abs{\frac{1}{n}\sum_{i=1}^{n}{\by_i}} + \norm[2]{ \frac{1}{n}\sum_{i=1}^{n}{\by_i \bx_i} } \right)
    .
\end{align*}
By Lemma~\ref{lemma:concentration-sampled-parity} and our assumptions on $n$, we bound the two data-dependent terms with probability at least $1 - \frac\delta3$ for some absolute constant $C_2$:
\begin{align*}
    \abs{\empinnerprod{ v^\T x + c, \parity}}
    &\leq 8\max\{\abs{\mu}, 1\}\paren{\sqrt{\frac{2 \ln(12/\delta)}{n}} + 2 \sqrt{\frac{d}n}} \\
    &\leq \frac{C_2}{d}\max\{\abs{\mu}, 1\}.
\end{align*}

Combining both bounds, we have with probability at least $1-\delta$,
\begin{align*}
    \alpha \le \empinnerprod{\bar{g}_{\mu}(x) + v^{\T}x + c, \parity} \le \frac{C_2}{d}\left(\abs{\mu} + \max\{\abs{\mu},1\} \right) \le \frac{2C_2}{d}\max\{\abs{\mu}, 1\}.
\end{align*}
Therefore, we conclude
\[
    \abs{\mu} \geq \frac{\alpha d}{2C_2} - 1 .
    \qedhere
\]
\end{proof}

\begin{lemma}\label{lemma:concentration-sampled-parity}
    Fix any $\delta \in (0, 1)$.
    Assume $n \geq O(\log(1/\delta))$ and
    $n = \omega(d)$, let $\setl{(\bx_i, \by_i)}_{i \in [n]}$ be the sampled parity dataset (where $\by_i = \parity(\bx_i)$ for all $i \in[n]$), and let $\bX \in \R^{n \times d}$ be the data matrix containing all samples.
    All of the following hold with probability $1 - \delta$:
    \begin{enumerate}[label=(\roman*)]
        \item  $\abs{\frac{1}{n}\sum_{i=1}^{n}{\by_i}} \le \sqrt{\frac{2\ln(4/\delta)}{n}}$;
        \item $\norm[2]{\frac{1}{n}\sum_{i=1}^{n}{\bx_i}} \le 2 \sqrt{\frac{d}{n}}$;
        \item  $\norm[2]{\frac{1}{n}\sum_{i=1}^{n}{\by_i \bx_i}} \le 2 \sqrt{\frac{d}{n}}$; and
        \item $\frac34\sqrt{n} \le \sigma_{d}(\bX) \le \sigma_{1}(\bX) \le 2\sqrt{n}.$
    \end{enumerate}
\end{lemma}
\begin{proof}
    Claim \textit{(i)} holds with probability at least $1 - \frac\delta2$ as a result of a standard application of Hoeffding's inequality (Lemma~\ref{lemma:hoeffding}) to a sum of Rademacher random variables.

    Claim \textit{(iv)} also holds with probability at least $1 - \frac\delta2$, since Lemma~\ref{lemma:random-matrix-singular-values} and the assumptions on $n$ imply that \[\sigma_1(\bX) \leq \sqrt{n} + C\paren{\sqrt{d} + \sqrt{\ln\frac2\delta}} \leq 2\sqrt{n}\] and \[\sigma_d(\bX) \geq \sqrt{n} - C\paren{\sqrt{d} - \sqrt{\ln\frac2\delta}} \geq \frac34\sqrt{n}.\]

    Claims \textit{(ii)} and \textit{(iii)} follow from the singular value bounds on $\bX$.
    \begin{align*}
    \norm[2]{\frac{1}{n}\sum_{i=1}^{n}{\bx_i}} &\le \frac1n \sqrt{\tr(\bX^\T \bX)} \le \frac1n \cdot 2\sqrt{nd} =  2 \sqrt{\frac{d}{n}}; \\
    \norm[2]{\frac{1}{n}\sum_{i=1}^{n}{\by_i \bx_i}} &\leq \frac1n \sqrt{\by^\T \bX \bX^\T \by} \leq \frac1n \cdot \sigma_1(\bX) \sqrt{d} \leq 2 \sqrt{\frac{d}{n}}.\qedhere
    \end{align*}
\end{proof}

\begin{lemma}\label{lemm:linear-term-bound}
Fix any
$\delta \in (0, 1)$.
Assume $n \geq O(\log(1/\delta))$ and
$n = \omega(d)$.
With probability at least $1 - \delta$ over the random measure $\nun$, if $\mu \in \measures$ satisfies
$\empnorm{g_{\mu}(\bx) + c + v^\T \bx - \parity(\bx) } \le 1$,
then $\max\{\abs{c + g_{\mu}(0)}, \norm[2]{v}\} \le 8\max\{\abs{\mu}, 1\}.$
\end{lemma}
\begin{proof}
We draw inspiration from the fact that the full parity dataset is orthogonal to any linear term and can never be well-approximated with large linear components. 
In other words, the square loss on approximating the full parity dataset with a linear function is minimized by the constant-zero function and strictly worsens as the linear terms increase.
That is, orthogonality ensures that $\cubenormtwol{c + v^\T x - \parity(x)}^2 = 1 + \absl{c}^2 + \norml[2]{v}^2$.
Thus, having an upper bound on the squared error imposes similar upper bounds on the norms of the linear terms.
We make a similar argument for the \textit{sampled} parity dataset, where we replace $\nu$ with $\nun$.

Without loss of generality, we incorporate $g_{\mu}(0)$ into $c$ and define $\bar{g}_\mu(x) = g_{\mu}(x) - g_{\mu}(0)$ which can be also represented as $\bar{g}_{\mu}(x) = \int{\bar{r}_{w,b}(x) \mu(dw,db)}$ where $\bar{r}_{w,b} = \relu(w^\T x + b) - \relu(b)$.
Let $\bX \in \R^{n \times d}$ be the collection of samples $\bx_i$ and let $\by_i = \parity(\bx_i)$.
We bound the squared loss of the linear component $v^\T x + c$, ignoring the neural network $\bar{g}_\mu$:
\begin{align*} 
\empnorm{c + v^\T x - \parity(x)}^2 
&= 1 + c^{2} + v^{\T}\paren{\frac{1}{n}\sum_{i=1}^{n}{\bx_i \bx_i^{\T}}}v - \frac{2}{n}v^{\T}\paren{\sum_{i=1}^{n}{(\by_i-c) \bx_i }} -  \frac{2c}{n}\sum_{i=1}^{n}\by_i \\
&\ge  1 + c^{2} + \frac1n \norm[2]{v}^{2} \sigma_{d}(\bX)^2 \\
&\quad - 2\norm[2]{v} \left( \norm[2]{\frac{1}{n}\sum_{i=1}^{n}{\by_i \bx_i}} + \abs{c}\norm[2]{\frac{1}{n}\sum_{i=1}^{n}{\bx_i}} \right) 
 - 2\abs{c}\abs{\frac{1}{n}\sum_{i=1}^{n}{\by_i}}.
\end{align*}

With probability $1 - \delta$, all events of Lemma~\ref{lemma:concentration-sampled-parity} hold, and we use them to lower-bound the squared loss.
\begin{align*}
\empnorm{c + v^\T x - \parity(x)}^2
&\ge 1 + c^{2} +\frac{9}{16}\norm[2]{v}^{2} - 4\sqrt{\frac{d}{n}}(1 + \abs{c})\norm[2]{v} - \frac{2 \sqrt{2 \ln(8 / \delta)}}{\sqrt{n}}\abs{c} \\
&\ge \frac{1}{4} \max\{\abs{c}, \norm[2]{v}\}^2.
\end{align*}
where we have used the assumptions on $n$ and the AM/GM inequality.
We now provide upper bounds on the square loss based on measure $\mu$ using the triangle inequality:
\[
\empnorm{c + v^\T x - \parity(x)} \le \empnorm{\bar{g}_{\mu}} + \empnorm{\bar{g}_{\mu}(x) + c + v^{\T}x - \parity(x)} \le \empnorm{\bar{g}_{\mu}} + 1.
\]
We may now connect $L^2(\nun)$ norm of $\bar{g}_{\mu}$ to its variational norm.
We bound the output of $\bar{g}_\mu$ on a single input $\bx_i$ by employing Cauchy-Schwarz:
\begin{align*}
 \bar{g}_{\mu}(\bx_i)^2 
 &\le \paren{\int{|\bar{r}_{w,b}(\bx_i)| \abs{\mu}(\dd w,\dd b)}}^2 \\
&\le \abs{\mu} \int{\bar{r}_{w,b}(\bx_i)^2 \abs{\mu}(\dd w,\dd b)}. 
\end{align*}

We sum over all $i$ to bound the norm of $\bar{g}_\mu$:
\begin{align*}
\empnorm{\bar{g}_{\mu}(x)}^2 
&\le \abs{\mu} \int{\empnorm{\bar{r}_{w,b}(x)}^2 } \abs{\mu}(\dd w,\dd b) 
\le \abs{\mu}^2 \sup_{\substack{w \in \sph, |b| \le \sqrt{d}}}{\empnorm{\bar{r}_{w,b}}^2 } \\
&\le \abs{\mu}^2 \sup_{w \in \sph}{\frac{1}{n}\sum_{i=1}^{n}{|w^{\T}\bx_i|^2 }} 
= \abs{\mu}^2 \frac{\sigma_1(\bX)^2}n \le 4 \abs{\mu}^2.
\end{align*} 
The second inequality relies on the Lipschitzness of $\relu$. Combining all the above,
\[
\frac{1}{2} \max\{\abs{c}, \norm[2]{v}\} \le \empnorm{c + v^\T x - \parity(x)} \le 1 + \empnorm{g_{\mu}} < 2 + 2 \abs{\mu}.\qedhere
\]
\end{proof}

\begin{lemma}\label{lemm:nonlinear-term-bound}
For $\bar{r}_{w,b}(x) = \relu(w^{\T} x + b)-\relu(b)$ and $n \geq d$,  there exists an absolute constant $C$ such that for any $\delta \in (0,1)$ with probability at least $1-\delta$,
\[
\abs{ \empinnerprod{\bar{r}_{w,b} , \parity} } \le \abs{ \cubeinnerprod{\bar{r}_{w,b} , \parity} } + C\sqrt{\frac{d \parenl{ \ln{n} + \ln\parenl{\fracl{1}{\delta}}}}{n}   },
\]
for all $w \in \sph, b \in [-\sqrt{d},\sqrt{d}]$.
\end{lemma}
\begin{proof}
Observe that the inner product over the sampled parity dataset is an unbiased estimate of the inner product over the full parity dataset,
$$
\EE{\empinnerprod{\bar{r}_{w,b}, \parity}} = \cubeinnerprod{ \bar{r}_{w,b}, \parity}.
$$
Let $\bZ_{w,b}$ denote the deviation from the mean, i.e.
$$\bZ_{w,b} = \empinnerprod{\bar{r}_{w,b}, \parity}  - \cubeinnerprod{ \bar{r}_{w,b}, \parity}.$$
We use standard concentration of measure techniques for the following steps:
\begin{enumerate}
    \item \label{itm:Lipschitz-process} $\bZ_{w,b}$ is Lipschitz in terms of its parameterization $(w,b)$ in the sense that $\abs{\bZ_{w_1,b_1} - \bZ_{w_2,b_2}} \le 4 \sqrt{d} \gamma((w_1,b_1), (w_2,b_2))$, where $\gamma$ is a distance defined later on. 
    \item \label{itm:subgaussian} $\bZ_{w,b}$ is $O(\frac{1}{\sqrt{n}})$-subgaussian for fixed $w,b$.
    \item \label{itm:bound-on-expectation} $\EE{\sup_{w \in \sph, b\in [-\sqrt{d},\sqrt{d}]}{\abs{\bZ_{w,b}}}} = O(\sqrt{\frac{d}{n}})$ using a covering argument.
    \item \label{itm:bounded-difference} The maximum of $\abs{\bZ_{w,b}}$ is close to its expectation due to the bounded difference inequality.
\end{enumerate}

\textbf{(Step~\ref{itm:Lipschitz-process})} 
Using the fact that $\relu$ is 1-Lipschitz and triangle inequality,
\begin{align*}\abs{\bZ_{w_1,b_1} - \bZ_{w_2,b_2}} 
&\leq \abs{\empinnerprod{\bar{r}_{w_1,b_1}, \parity}  - \empinnerprod{\bar{r}_{w_2,b_2}, \parity}} + \abs{\cubeinnerprod{ \bar{r}_{w_1,b_1}, \parity}   - \cubeinnerprod{ \bar{r}_{w_2,b_2}, \parity}} \\
&\le 2 \cubenorminfty{ \bar{r}_{w_1,b_1} - \bar{r}_{w_2,b_2}}  \\ 
&\le 2 ( \cubenorminfty{ \bar{r}_{w_1,b_1} - \bar{r}_{w_2,b_1} } + \cubenorminfty{\bar{r}_{w_2,b_1} - \bar{r}_{w_2,b_2}}) \\
&\le 2\paren{ \max_{x \in \flip^d}(w_1-w_2)^{\T}x  + 2 \abs{b_1 - b_2} } \\
&\le 4\sqrt{d} \left( \norm[2]{w_1 - w_2} + \frac{|b_1 - b_2|}{\sqrt{d}} \right) \eqqcolon 4\sqrt{d} \gamma((w_1,b_1), (w_2,b_2)).
\end{align*}
Thus $\bZ_{w,b}$ is $4 \sqrt{d}$-Lipschitz with respect to $\gamma$.

\textbf{(Step~\ref{itm:subgaussian})}
We bound the subgaussianity of $\bZ_{w,b}$.
\begin{align*}
    \norm[\psi_2]{\bZ_{w,b}} 
    &\leq C_1 \norm[\psi_2]{ \empinnerprod{ \bar{r}_{w,b}, \parity} } 
    = C_1 \norm[\psi_2]{\sum_{i=1}^n \by_{i}(\relu(w^{\T}\bx_i + b)- \relu(b)) } \\
    &\le \frac{C_2}{\sqrt{n}} \norm[\psi_2]{ \by_{1}(\relu(w^{\T}\bx_1 + b)- \relu(b)) } \\
    &\le  \frac{C_2}{\sqrt{n}} \norm[\psi_2]{ \relu(w^{\T}\bx_1 + b) - \relu(b)} \\
    &\leq \frac{C_2}{\sqrt{n}} \norm[\psi_2]{w^\T \bx_1}
    \leq \frac{2C_2}{\sqrt{n}} 
\end{align*}
The first, second, and fourth inequalities rely on the centering, averaging, and Lipschitzness properties of subgaussian random variables in Lemma~\ref{lemma:subg}.
The third inequality follows from $\abs{\by_1} = 1$, and the final is due to the 2-subgaussianity of a vector with i.i.d.~Rademacher components.

\textbf{(Step~\ref{itm:bound-on-expectation})}
Let $\mathcal{N}_\epsilon$ be an $\epsilon$-covering of $\sph \times [-\sqrt{d}, \sqrt{d}]$ with respect to $\gamma$.
We bound its size using the standard $\epsilon$-net result in Lemma~\ref{lemma:covering} for $\epsilon \leq 2$.
\begin{align*}
\mathcal{N}\left( \epsilon, \mathbb{S}^{d-1} \times [-\sqrt{d},\sqrt{d}], \gamma \right) &\le \mathcal{N}\left( \frac{\epsilon}{2}, \mathbb{S}^{d-1} , \|\cdot\|_{2} \right) \times \mathcal{N}\left( \frac{\epsilon}{2},[-1,1], |\cdot| \right) \\
&\le \paren{\frac{6}{\epsilon}}^{d} \cdot \frac{4}{\epsilon} \le \paren{\frac{6}{\epsilon}}^{d+1}.
\end{align*}

We bound the expected maximum deviation over all $w$ and $b$ by employing a bound on the expected maximum of subgaussian random variables (Lemma~\ref{lemma:subg}), applying the covering numbers argument, letting $\pi(w, b) = \argmin_{(w', b') \in \mathcal{N}_\epsilon} \gamma((w, b), (w', b'))$, and setting $\epsilon \coloneqq 1 / \sqrt{n}$.
\begin{align*}
\EE{ \sup_{\substack{w \in \mathbb{S}^{d-1}, b\in [-\sqrt{d},\sqrt{d}] }}{ \abs{\bZ_{w,b}} } } 
&\le \EE{ \sup_{w,b}{ \abs{ \bZ_{w,b} - \bZ_{\pi(w,b)} } } } + \EE{ \sup_{(w,b) \in \mathcal{N}_{\epsilon}}{ \abs{\bZ_{w,b}} } }  \\
&\le 4\sqrt{d}\epsilon + \frac{2C_2}{\sqrt{n}} \sqrt{\ln \mathcal{N}\left( \epsilon, \mathbb{S}^{d-1} \times [-\sqrt{d},\sqrt{d}], \gamma \right)} \\
&\le 4\sqrt{d}\epsilon + 2C_2\sqrt{\frac{d+1}{n} \ln \frac{6}{\epsilon}} 
\leq C_3 \sqrt{\frac{d \ln n}n}.
\end{align*}

\textbf{(Step~\ref{itm:bounded-difference})}
We conclude by showing that $\sup_{w, b} \abs{\bZ_{w,b}}$ is close to its expectation with high probability due to the McDiarmid's inequality (Lemma~\ref{lemma:mcdiarmid}).
Consider a perturbation where $\bx_i$ is replaced by some $\bx_i' \in \flip^d$ with $\by_i' = \parity(\bx_i')$, and let $\bZ_{w,b}^i$ denote the resulting deviation term.
\begin{align*}
    \abs{\sup_{w, b} \abs{\bZ_{w,b}} - \sup_{w, b} \abs{\bZ_{w, b}^i}}
    &\leq \sup_{w, b} \abs{\bZ_{w,b} - \bZ_{w,b}^i} 
    = \frac1n \sup_{w, b}  \abs{\by_i \bar{r}_{w, b}(\bx_i) - \by_i' \bar{r}_{w, b}(\bx_i')} \\
    &\leq\frac1n \sup_{w, b}  \bracket{\abs{\bar{r}_{w, b}(\bx_i) - \bar{r}_{w, b}(\bx_i')} + \abs{(\by_i - \by_i')\bar{r}_{w, b}(\bx_i)}} \\
    &\leq \frac1n\bracket{\norm[2]{\bx_i - \bx_i'} + 2 \norm[2]{\bx_i}} 
    \leq \frac{4\sqrt{d}}n
\end{align*}

Hence, with probability at least $1 - \delta$:
\[\sup_{w, b} \abs{\bZ_{w, b}} 
\leq \sqrt{\frac{8d\ln 1/\delta}{n}} + \EE{\sup_{w, b} \abs{\bZ_{w, b}}} 
\leq C_4 \sqrt{\frac{d(\ln n + \ln 1/\delta)}{n}}.\]
The bound in the lemma statement immediately follows.
\end{proof}
 \section{Proofs for Section~\ref{sec:parity-gen}}\label{asec:parity-gen}

\subsection{Proofs for Section~\ref{ssec:parity-gen-lb}}\label{assec:parity-gen-lb}

\begin{lemma}
  \label{lemma:emprademacher-conditionally-subgaussian}
  Fix $S \subseteq [d]$ with $|S| \geq 3$, and let $\bx \sim \unif(\flip^d)$.
  Conditional on the value of $\parity_S(\bx)$, the random vector $\bx$ is mean-zero, isotropic, and satisfies
  \begin{equation*}
    \E{ \exp(u^\T \bx) \mid \parity_S(\bx) } \leq \exp(\norml[2]{u}^2)
  \end{equation*} 
  for all $u = (u_1,\dotsc,u_d) \in \R^d$.
\end{lemma}
\begin{proof}
  The assumption $|S| \geq 3$ implies that, conditioned on $\parity_S(\bx)$, the $\{ \bx_i \}_{i \in [d]}$ are mean-zero and pairwise uncorrelated.
So it remains to show that, for any vector $u = (u_1,\dotsc,u_d) \in \R^d$,
  \begin{equation*}
    \E{ \exp(u^\T \bx) \mid \parity_S(\bx) } \leq \exp(\norml[2]{u}^2) .
  \end{equation*}
  So fix $u$, and fix any $i \in S$.
Let $u_{-i}$ (respectively, $\bx_{-i}$) be the vector obtained from $u$ (respectively, $\bx$) by removing the $i$-th entry.
  Observe that $\bx_{-i} \mid \parity_S(\bx) \sim \unif(\flip^{d-1})$, and also that $\bx_i \mid \parity_S(\bx) \sim \unif(\flip)$.
  (But, of course, $\bx_{-i}$ and $\bx_i$ are not conditionally independent given $\parity_S(\bx)$.)
  Therefore, using Cauchy-Schwarz,
  \begin{align*}
    \E{ \exp(u^\T \bx) \mid \parity_S(\bx) }
    & = \E{ \exp\parenl{ u_{-i}^\T \bx_{-i} } \exp\parenl{ u_i \bx_i } \mid \parity_S(\bx) } \\
    & \leq
    \sqrt{ \E{ \exp\parenl{2 u_{-i}^\T \bx_{-i} } \mid \parity_S(\bx) } }
    \sqrt{ \E{ \exp\parenl{2 u_i \bx_i } \mid \parity_S(\bx) } } \\
    & \leq
    \sqrt{ \exp\parenl{\norml[2]{2 u_{-i}}^2/2} }
    \sqrt{ \exp\parenl{(2u_i)^2 / 2 } }
    \\
    & = \exp\parenl{ \norml[2]{u}^2 } .
  \end{align*}
  Above, the second inequality uses the moment generating function bound from Lemma~\ref{lemma:mgf-flip}, as well as the conditional independence of $\{ \bx_j : j \neq i \}$ given $\parity(\bx)$.
\end{proof}

\lemmasmallsampleparityapprox*

\begin{proof}
  Throughout, we take $C>0$ to be a suitably large constant, and we assume $n \leq d^2/C$.
  The construction of $\bg \colon \domain \to \R$ is based on typical statistical behavior of the random examples $(\bx_1,\by_1), \dotsc, (\bx_n,\by_n)$, where $\by_i := \parity_S(\bx_i)$ for each $i \in [n]$.
  We may assume that $n \geq d$, since otherwise the examples can be perfectly fit with a linear function $\bg$, and this function has $\rnorm{\bg} = 0$.
  So, combining the assumption $n\geq d$ with the assumption
  $n \leq d^2/C$
implies that $d \geq C$.
  Observe that $\by_1, \dotsc, \by_n$ are i.i.d.~$\unif(\flip)$ random variables.
  Since $n \geq d \geq C$, it follows by standard binomial tail bounds that with probability at least $5/6$ over the realizations of $\by_1, \dotsc, \by_n$, the number of $\by_i$ that are equal to $1$ is at least $n/3$, and also that the number of $\by_i$ that are equal to $-1$ is also at least $n/3$.
  We henceforth condition on this ``good event'' (which depends only on $\by_1, \dotsc, \by_n$).

  To help define our construction of $\bg \colon \domain \to \R$ and set up the rest of the analysis, we partition $[n]$ into disjoint groups 
$G_1, G_2, \dotsc, G_m$ so that for each $j \in [m]$, (i) the size $n_j := |G_j|$ of the $j$-th group is between $c_1 d / \ln d$ and $2c_1 d / \ln d$, and (ii) all $\by_i$ for $i \in G_j$ are the same (i.e., all $+1$ or all $-1$).
  Here, with foresight, we set $c_1 := 1/256$; by using $d \geq C$, we ensure that each group is non-empty, and also that $n_j < d$.
  The feasibility of this partitioning is ensured because, in the ``good event'' (and using $d \geq C$), the number of $i \in [n]$ with $\by_i = 1$ is at least $n/3 \geq d/3 \geq c_1 d/\ln d$, and same for the number of $i \in [n]$ with $\by_i = -1$.
  Let $z^{(j)}$ denote the common $\by_i$ value for all $i \in G_j$.
  Finally, note that the number of groups $m$ satisfies $m \leq n \ln(d) / (c_1 d)$.

  We now define our construction of $\bg \colon \domain \to \R$.
  Let $\bA_j$ denote the random $n_j \times d$ matrix whose rows are the $\bx_i^\T$ for $i \in G_j$, and define the random vector $\bw^{(j)} := \bA_j^{\dag} (z^{(j)} \vec1)$.
  Observe that $\bw^{(j)}$ is a least squares solution to the system of linear equations $\{ \bx_i^\T w = \by_i : i \in G_j \}$, since $\by_i = z^{(j)}$ for all $i \in G_j$.
  We define $\bg$ as follows:
  \begin{equation*}
    \bg(x) = \sum_{j=1}^m z^{(j)} \relu(2 z^{(j)} \bw^{(j)\T} x - 1) .
  \end{equation*}

  To analyze our construction, we consider the realizations of $\bx_1, \dotsc, \bx_n$, and establish some basic properties that hold with sufficiently high probability (conditional on the ``good event'').
  Note that within a group $G_j$, the $\{ \bx_i \}_{i \in G_j}$ are (conditionally) iid, and the realizations across groups are also (conditionally) independent.

  We claim that with probability at least $5/6$ (conditional on the ``good event''),
  \begin{itemize}
    \item (P1) $\bw^{(j)\T} \bx_i = \by_i$ for all $j \in [m]$ and $i \in G_j$;
    \item (P2) $\norml[2]{\bw^{(j)}} \leq 2\sqrt{n_j/d}$ for all $j \in [m]$.
  \end{itemize}
  To establish this claim, we lower-bound the $n_j$-th largest singular value $\sigma_{n_j}(\bA_j)$.
Note that $\sigma_{n_j}(\bA_j)$ is at least the corresponding singular value of the $n_j \times (d-1)$ submatrix $\bB_j$ obtained from $\bA_j$ by removing the $t$-th column of $\bA_j$ for some $t \in S$.
  (If $S$ is empty, we may remove any column.)
  Since the rows of $\bA_j$ are independent, and since the entries of $\bx_i$ after removing the $t$-th one are iid $\unif(\flip)$ random variables, it follows that the $n_j \times (d-1)$ entries of $\bB_j$ are iid $\unif(\flip)$ random variables.
  Hence, the rows of $\bB_j^\T$ are independent, mean-zero, isotropic, and $O(1)$-subgaussian.
  By Lemma~\ref{lemma:random-matrix-singular-values} and a union bound, with probability at least $1 - 2m\exp(-\min_{j \in [m]} \{ n_j \})$,
  \begin{equation*}
    \sigma_{n_j}(\bA_j)
    \geq \sigma_{n_j}(\bB_j^\T)
    \geq \sqrt{d-1} - C_2 \sqrt{n_j}
    \geq \paren{ \sqrt{1 - \frac{1}{d}} - C_2 \sqrt{\frac{c_1}{\ln d}} } \sqrt{d}
    \quad \text{for all $j \in [m]$}
    ,
  \end{equation*}
  where $C_2>0$ is twice the absolute constant from Lemma~\ref{lemma:random-matrix-singular-values}, and the final inequality uses the upper-bound on $n_j$.
  The fact $d \geq C$ and the upper-bounds on $m$ and $n$ altogether imply that the probability of the above event is at least $5/6$, and also that $\sqrt{1 - 1/d} - C_2\sqrt{c_1/\ln d} \geq 1/2$.
  So in this event, for each $j \in [m]$, the column space of $\bA_j$ has rank $n_j$, so the system of linear equations defining $\bw^{(j)}$ is feasible, and
  \begin{equation*}
    \norml[2]{\bw^{(j)}}
    = \norml[2]{\bA_j^\dag (z^{(j)} \vec1)}
    \leq \sigma_1(\bA_j^\dag) \norml[2]{\vec1}
    = \frac{\sqrt{n_j}}{\sigma_{n_j}(\bA_j)}
    \leq 2\sqrt{\frac{n_j}{d}} .
  \end{equation*}
  This establishes P1 and P2 in the event as claimed.

  We further claim that with probability at least $5/6$ (conditional on the ``good event''),
  \begin{itemize}
    \item (P3) $|\bw^{(j)\T} \bx_i| \leq 4\norml[2]{\bw^{(j)}} \sqrt{\ln d}$ for all $j \in [m]$ and $i \in [n] \setminus G_j$.
  \end{itemize}
  To establish this claim, first observe that $\bx_i$ and $\bw^{(j)}$ are independent for $i \notin G_j$. 
Moreover, by Lemma~\ref{lemma:emprademacher-conditionally-subgaussian}, conditional on $\bw^{(j)}$ (with $G_j \not\ni i$), $\bx_i^\T \bw^{(j)}$ is a mean-zero random variable satisfying
  \begin{equation*}
    \E{ \exp(\bw^{(j)\T} \bx) \mid \parity_S(\bx), \bw^{(j)} } \leq \exp(\norml[2]{\bw^{(j)}}^2) .
  \end{equation*}
  So, by Markov's inequality and a union bound, we have with probability at least $5/6$,
  \begin{equation*}
    |\bw^{(j)\T} \bx_i| \leq (\sqrt{2} \norml[2]{\bw^{(j)}}) \sqrt{2 \ln (12mn)}
    \quad \text{for all $j \in [m]$ and $i \in [n] \setminus G_j$} .
  \end{equation*}
  Using $d \geq C$ and the upper-bounds on $m$ and $n$, we obtain $\sqrt{\ln(12mn)} \leq 2\sqrt{\ln d}$, and hence we deduce P3 from the above inequality.

  So, by a union bound, with probability at least $2/3$ (conditional on the ``good event''), the properties P1, P2, and P3 all hold simultaneously.
  We can now establish the desired properties of $\bg$.
  Using $d \geq C$, P2, and the upper-bounds on $m$ and $n$, we obtain
  \begin{equation*}
    \rnorm{\bg}
    \leq 2\sum_{j=1}^m \norml[2]{\bw^{(j)}}
    \leq 4\sum_{j=1}^m \sqrt{\frac{n_j}{d}}
    \leq 4\sqrt{m \sum_{j=1}^m \frac{n_j}{d}}
    = 4\sqrt{\frac{mn}{d}}
    \leq \frac{4n\sqrt{\ln d}}{d}
.
  \end{equation*}
  Furthermore, by P1, we have for any $j \in [m]$ and $i \in G_j$,
  \begin{equation*}
    2z^{(j)} \bw^{(j)\T} \bx_i - 1 = 2 z^{(j)} \by_i - 1 = 1 .
  \end{equation*}
  And by P2, P3, and the upper-bound on $n_j$, we have for any $j \in [m]$ and $i \in [n] \setminus G_j$,
  \begin{equation*}
    2z^{(j)} \bw^{(j)\T} \bx_i - 1
    \leq 2 |\bw^{(j)\T} \bx_i| - 1
    \leq 16 \sqrt{\frac{n_j \ln d}{d}} - 1
    \leq 16 \sqrt{c_1} - 1
    = 0 ,
  \end{equation*}
  and hence $\relu(2z^{(j)} \bw^{(j)\T} \bx_i - 1) = 0$.
  Therefore, for any $i \in [n]$, if $i \in G_j$,
  \begin{equation*}
    \bg(\bx_i) = z^{(j)} \relu(2z^{(j)} \bw^{(j)\T} \bx_i - 1) = z^{(j)} = \by_i .
    \qedhere
  \end{equation*}
\end{proof}

\subsection{Proofs for Section~\ref{ssec:parity-gen-ub}}\label{assec:parity-gen-ub}

\thmparitygenub*

\begin{proof}
Let $\bG$ denote all solutions to
\eqref{interpolating-r-norm} on the sampled parity dataset, so $\empnorm{\bg - \parity} = 0$ for all $\bg \in \bG$.
By Proposition~\ref{prop:integral-representation}, we can write each $\bg$ as $\bg(x) = g_{\bmu}(x) + \bv^\T x + \bc$, where $\bmu \in \measures$, $\bv \in \R^d$, and $\bc \in \R$.
Furthermore, we can assume that $g_{\bmu}(0) = 0$ by absorbing the value of $g_{\bmu}(0)$ into $\bc$ (at the cost of losing the evenness of $\bmu$, but evenness is not needed in the sequel).
Lemma~\ref{lemma:rnorm-relu-l1}, Theorem~\ref{thm:sparse-vp-solution}, and Theorem~\ref{thm:parity-rnorm-ub} together imply that every $\bg \in \bG$ satisfies $\rnorm{\bg} \leq Cd$ for some absolute constant $C>0$.
Let $\mathcal{E}$ be the event that $\max\{\absl{\bc}, \norml[2]{\bv}\} \leq 8Cd$ (for all $\bg \in \bG$), and let $\mathcal{E}^{\operatorname{c}}$ be its complement; event $\mathcal{E}$ occurs with probability at least $1-\delta/2$ by Lemma~\ref{lemm:linear-term-bound}, for another absolute constant $C'>0$.

Since, for each $\bg \in \bG$, we have $\bg(\bx_i) = \parity(\bx_i)$ for every example $(\bx_i,\parity(\bx_i))$ in the sampled parity dataset, it follows that $\empnorm{\clip \circ\, \bg - \parity} = \empnorm{\bg - \parity} = 0$ for all such $\bg \in \bG$.
For any $t>0$,
\begin{align*}
    \lefteqn{
        \pr{ \sup_{\bg \in \bG} \cubenormtwo{\clip \circ\, \bg - \parity}^2 \geq t }
    } \\
    & \leq
    \pr{ \sup_{\bg \in \bG} \cubenormtwo{\clip \circ\, \bg - \parity}^2 \geq t \mid \mathcal{E} }
    + \pr{\mathcal{E}^{\operatorname{c}}}
    \\
    & =
    \pr{
        \sup_{\bg \in \bG}
        \cubenormtwo{\clip \circ\, \bg - \parity}^2 - \empnorm{\clip \circ\, \bg - \parity}^2 \geq t \mid \mathcal{E}
    }
    + \pr{\mathcal{E}^{\operatorname{c}}}
    \\
    & \leq
    \pr{
        \sup_{g \in \mathcal{G}_0}
        \cubenormtwo{\clip \circ\, g - \parity}^2 - \empnorm{\clip \circ\, g - \parity}^2
        \geq t
    }
    + \delta/2
    ,
\end{align*}
where
\[
    \mathcal{G}_0 :=
    \left\{
        x \mapsto g(x) + v^\T x + c : \rnorm{g} \leq Cd , \ \max\{\norml[2]{v}, \absl{c} \} \leq 8Cd
    \right\}
    .
\]
Define
\[
    t_0 := 4 \underbrace{\EE{ \sup_{g \in \mathcal{G}_0} \frac1n \sum_{i=1}^n \beps_i g(\bx_i) }}_{\Rad_n(\mathcal{G}_0)}
    \, + \, 4\sqrt{\frac{\log(2/\delta)}{n}} .
\]
Above, $\Rad_n(\mathcal{G}_0)$ denotes the Rademacher complexity of $\mathcal{G}_0$, where $\beps_1,\dotsc,\beps_n \simiid \unif(\flip)$, independent of $\bx_1,\dotsc,\bx_n$.
Since, for any $y \in \flip$, the mapping $z \mapsto (y - \clip(z))^2 = (1-y \clip(z))^2$ is $4$-Lipschitz and has range $[-4,4]$, it follows by standard Rademacher complexity arguments~\citep[see, e.g.,][Theorem 8]{meir2003generalization} that
\begin{equation*}
    \pr{
        \sup_{g \in \mathcal{G}_0}
        \cubenormtwo{\clip \circ\, g - \parity}^2 - \empnorm{\clip \circ\, g - \parity}^2
        \geq
        t_0
    } \leq \delta / 2
    .
\end{equation*}

So it remains to show that $t_0 \leq \epsilon$ under the assumption $n \geq C_0 ((d^3 + \log(1/\delta))/\epsilon^2)$ for suitably large absolute constant $C_0>0$.
The second term in the definition of $t_0$ is at most $\epsilon/2$ provided that $C_0$ is chosen large enough.
To bound the first term ($\Rad_n(\mathcal{G}_0)$), we use the fact that
\[
    \Rad_n(\mathcal{G}_0)
    = \Rad_n(\mathcal{G}_1)
    + \Rad_n(\mathcal{G}_2)
\]
where $\mathcal{G}_1 := \{ g : \rnorm{g} \leq Cd \}$ and $\mathcal{G}_2 := \{ x \mapsto v^\T x + c : \max\{\norml[2]{v} , \absl{c} \} \leq 8Cd \}$.
Theorem~10 of \citet{pn21a} implies
\[
    \Rad_n(\mathcal{G}_1)
    \leq \frac{2 \cdot (Cd) \cdot \sqrt{d}}{\sqrt{n}}
    = O\left( \sqrt{\frac{d^3}{n}} \right) ,
\]
while Theorem~3 of \citet{kakade2008complexity} implies
\[
    \Rad_n(\mathcal{G}_2)
    \leq \sqrt{d+1} \cdot \sqrt{(8Cd)^2} \cdot \sqrt{\frac{2}{n}}
    = O\left( \sqrt{\frac{d^3}{n}} \right) .
\]
By choosing $C_0$ large enough, it follows that $\Rad_n(\mathcal{G}_0) \leq \epsilon/8$.
Hence, we have shown that $t_0 \leq \epsilon$ as required.
\end{proof}

 \section{Proofs for Section~\ref{sec:cosine-approx}}\label{asec:cosine-approx}

\subsection{Proof of Theorem~\ref{thm:periodic-ub}}
\label{assec:periodic-rnorm-ub}

\begin{restatable}[Detailed version of Theorem~\ref{thm:periodic-ub}]{theorem}{thmperiodicubfull}\label{thm:periodic-ub-detailed}
  Suppose $f \colon \domain \to [-1,1]$ is given by $f(x) = \phi(v^\T x)$ for some unit vector $v \in \sph$ and some $\phi \colon [-\sqrt{d}, \sqrt{d}] \to [-1,1]$ that is $L$-Lipschitz and $\rho$-periodic for $\rho \in [\norm[\infty]{v},1]$.
  Let $\sigma_{\rho,v} := \sqrt{2\rho\norml[1]{v} - 1}$, and fix any $\epsilon \in (0,1)$.
  There exists a function $g \colon \domain \to \R$ such that the following properties hold:
  \begin{enumerate}
    \item $\absl{f(x) - g(x)} \leq \epsilon$ for all $x \in \flip^d$;
    \item $g$ is represented by a neural network of width at most
      \begin{equation*}
        O\paren{\frac{dL(\sigma_{\rho,v} \sqrt{\log(1/\epsilon)} + \rho \log(1/\epsilon))}{\epsilon^2}} ;
      \end{equation*}
    \item $g$ satisfies
      \begin{equation*}
       \rnorm{g} =
        O\paren{ \frac{L^2(\sigma_{\rho,v} \sqrt{\log(1/\epsilon)} + \rho \log(1/\epsilon))(\sigma_{\rho,v} + \sqrt{\log(d/\epsilon)})}{\epsilon} } .
      \end{equation*}
  \end{enumerate}
\end{restatable}

\begin{proof}
  We first describe the (randomized) construction of our approximating neural network $\bg \colon \R^d \to \R$.
  For $w \in \Z^d$, define $h_w \colon \R^d \to \R$ by $h_w(x):=\phi(v^\T x + \rho w^\T x)$.
  Let $\bw \in \Z^d \setminus \{-(1/\rho)v\}$ be a random vector with distribution to be specified later in the proof.
  Let $\bw^{(1)}, \dotsc, \bw^{(k)}$ be i.i.d.~copies of $\bw$ for a positive integer $k > (9(d+1)\ln(2))/\epsilon$, and let $\bh^{(j)} := h_{\bw^{(j)}}$ for each $j$.
  Observe that each $\bh^{(j)}$ can be written as $\bh^{(j)}(x) = \phi^{(j)}(x^\T \bu^{(j)})$ for
  \begin{equation*}
    \bu^{(j)} := \frac1{\norml[2]{v+\rho\bw^{(j)}}} (v + \rho\bw^{(j)}) \in \sph
    \quad \text{and} \quad
    \phi^{(j)}(z) := \phi(\norml[2]{v+\rho\bw^{(j)}} z) ,
  \end{equation*}
  where $\phi^{(j)} \colon \R \to [-1,1]$ is $L_j$-Lipschitz for $L_j := L\norml[2]{v+\rho\bw^{(j)}}$ (using the $L$-Lipschitzness of $\phi$).
  Let $\tau>0$ be a value (depending on $\rho$, $v$, and $\epsilon$) also to be specified later.
  By Lemma~\ref{lemma:lipschitz-approx} (with $t := \tau/\norml[2]{v + \rho \bw^{(j)}}$ and $\delta := \epsilon/3$), there exist $\tilde\bh^{(1)}, \dotsc, \tilde\bh^{(k)}$ such that:
  \begin{itemize}
\item (H1) $\tilde\bh^{(j)} \colon \R^d \to \R$ is represented by a neural network of width at most $O(\tau L/\epsilon)$;
    \item (H2) $\rnorml{\tilde\bh^{(j)}} = O(\tau L^2 \norml[2]{v + \rho \bw^{(j)}}/\epsilon)$; 
    \item (H3) $|\tilde\bh^{(j)}(x) - \bh^{(j)}(x)| \leq \epsilon/3$ for all $x \in \flip^d$ such that $|x^\T \bu^{(j)}| \leq \tau / \norml[2]{v + \rho \bw^{(j)}}$;
    \item (H4) $|\tilde\bh^{(j)}(x) - \bh^{(j)}(x)| \leq 1$ for all $x \in \R^d$.
  \end{itemize}
  Our approximating neural network $\bg \colon \R^d \to \R$ is defined by
  \begin{equation*}
    \bg(x) := \frac1k \sum_{j=1}^k \tilde\bh^{(j)}(x) .
  \end{equation*}
  By construction and using properties H1 and H2 (above), the following properties of $\bg$ are immediate:
  \begin{itemize}
    \item (G1) $\bg$ is represented by a neural network of width at most $O(k\tau L/\epsilon)$;
    \item (G2) $\max\{\rnorml{\bg}, \vnorml{\bg} \} = O(\tau L^2 \max_{j \in [k]} \norml[2]{v + \rho \bw^{(j)}}/\epsilon)$.
  \end{itemize}
  Note that these properties are given in terms of $\tau$, which has yet to be specified, as well as $\max_{j \in [k]} \norml[2]{v + \rho \bw^{(j)}}$, which is a random variable.
  So, in the remainder of the proof, we choose a particular distribution for $\bw$ (and hence also for $\bw^{(1)},\dotsc,\bw^{(k)}$) and a value of $\tau$ that, together, will ultimately allow us to establish the existence of an approximating neural network with the desired properties via the probabilistic method.

  We first specify the probability distribution of $\bw$ and establish some of its properties.
  We let $\bw = (\bw_1,\dotsc,\bw_d)$ be a vector of independent random variables $\bw_1,\dotsc,\bw_d$ with $p_i := \prl{ \bw_i = -2\sign{v_i} } = \abs{v_i}/(2\rho)$ and $\prl{ \bw_i = 0 } = 1-p_i$.
  Note that $p_i \in [0,1/2]$ for all $i$ since we have assumed $\rho \geq \norm[\infty]{v}$, so the distribution of $\bw$ is well-defined.
  Furthermore, observe that $\bw \neq -(1/\rho) v$ almost surely (since $v \neq 0$ and $\rho \geq \norml[\infty]{v}$ by assumption), $\EEl{v + \rho\bw} = 0$, and
  \begin{equation*}
    \EE{\norm[2]{v + \rho \bw}^2}
    = \sum_{i=1}^d \Varl{\rho \bw_i}
    = \sum_{i=1}^d 4\rho^2 \cdot \frac{\abs{v_i}}{2\rho} \cdot \left( 1- \frac{\abs{v_i}}{2\rho} \right)
    = 2\rho \norm[1]{v} - \norm[2]{v}^2
    = \sigma_{\rho,v}^2 .
  \end{equation*}
  Moreover, $\norml[2]{v + \rho \bw}$ is a function of independent random variables $\bw_1,\dotsc,\bw_d$ that satisfies the $(2|v_1|,\dotsc,2|v_d|)$-bounded differences property.
  By McDiarmid's inequality (Lemma~\ref{lemma:mcdiarmid}) and Jensen's inequality,\begin{equation}
    \label{eq:cos-approx-mcdiarmid}
    \pr{ \norml[2]{v + \rho \bw} \geq \sigma_{\rho,v} + \sqrt{2\ln(1/\delta)} } \leq \delta \quad \text{for all $\delta \in (0,1)$} .
  \end{equation}
  Finally, for any fixed $x \in \flip^d$, $x^\T (v + \rho \bw) = \sum_{i=1}^d x_i (v_i + \rho \bw_i)$ is a sum of $d$ independent, mean-zero random variables, with variance $\Var{x^\T (v + \rho \bw)} = \sigma_{\rho,v}^2$ and $|x_i (v_i + \rho \bw_i)| \leq 2\rho$ almost surely for each $i$.
  By Bernstein's inequality (Lemma~\ref{lemma:bernstein}),\begin{equation}
    \label{eq:cos-approx-bernstein}
    \pr{ \abs{x^\T (v + \rho \bw)} \geq \sigma_{\rho,v} \sqrt{2 \ln(2/\delta)} + 2\rho\ln(2/\delta)/3 } \leq \delta
    \quad \text{for all $x \in \flip^d$, $\delta \in (0,1)$}
    .
  \end{equation}

  We now show that $\bg$ has the desired properties with positive probability.
  Since $w^\T x$ is an integer for any $w \in \Z^d$ and $x \in \flip^d$, and since $v + \rho \bw^{(j)} \neq 0$ almost surely, the $\rho$-periodicity of $\phi$ implies that $g_{\bw^{(j)}}(x) = f(x)$ for all $x \in \flip^d$ and all $j \in [k]$.
  Therefore, the intermediate (random) function $\bg_1 \colon \R^d \to [-1,1]$ defined by $\bg_1(x) := \frac1k \sum_{j=1}^k \bh^{(j)}(x)$ satisfies $\bg_1(x) = f(x)$ for all $x \in \flip^d$.
  For each $x \in \flip^d$, let \[\br(x) := |\{ j \in [k] : \absl{x^\T (v + \rho \bw^{(j)})} \geq \tau \}| = |\{ j \in [k] : \absl{x^\T \bu^{(j)}} \geq \tau / \norml[2]{v + \rho \bw^{(j)}} \}|.\]
  Using the approximation properties of $\tilde\bh^{(j)}$ (i.e., H3 and H4 from above), we have for each $x \in \flip^d$,
  \begin{equation*}
    \abs{\bg(x) - \bg_1(x)}
    = \frac1k \abs{ \sum_{j=1}^k (\tilde\bh^{(j)}(x) - \bh^{(j)}(x)) }
    \leq \left(1 - \frac{\br(x)}{k} \right) \cdot \frac\epsilon3 + \frac{\br(x)}{k} \cdot 1 .
  \end{equation*}
  This final expression is at most $\epsilon$ if $\br(x) \leq 2k\epsilon/3$.
  We choose $\tau$ such that for any $x \in \flip^d$, we have $\prl{ |x^\T (v + \rho \bw)| > \tau } \leq \epsilon/3$.
  By~\eqref{eq:cos-approx-bernstein}, it suffices to choose
  \begin{equation*}
    \tau := \sigma_{\rho,v} \sqrt{2 \ln(6/\epsilon)} + 2\rho\ln(6/\epsilon)/3 .
  \end{equation*}
  By a multiplicative Chernoff bound (Lemma~\ref{lemma:chernoff}) and a union bound over all $x \in \flip^d$,
  \begin{equation*}
    \pr{ \exists x \in \flip^d \ \text{s.t.} \ \br(x) > 2k\epsilon/3 } \leq 2^d \cdot e^{-k\epsilon/9} < \frac12 ,
  \end{equation*}
  where the final inequality uses the choice of $k > (9(d+1) \ln 2)/\epsilon$.
  Therefore, with probability more than $1/2$, we have $\br(x) \leq 2k\epsilon/3$ for all $x \in \flip^d$, and hence
  \begin{equation}
    \label{eq:cos-approx-accuracy}
    \absl{\bg(x) - f(x)}
    = \absl{\bg(x) - \bg_1(x)} \leq \epsilon
    \quad \text{for all $x \in \flip^d$} .
  \end{equation}
  Finally, by \eqref{eq:cos-approx-mcdiarmid} an a union bound over all $j \in [k]$, we have that with probability more than $1/2$,
  \begin{equation}
    \label{eq:cos-approx-norms}
    \max_{j \in [k]} \norml[2]{v + \rho \bw^{(j)}}
    \leq \sigma_{\rho,v} + \sqrt{2\ln(2k)}
    .
  \end{equation}
  So, there is a positive probability that both~\eqref{eq:cos-approx-accuracy} and~\eqref{eq:cos-approx-norms} hold simultaneously, and in this event, it can be checked (via G1 and G2 above) that the function $\bg$ satisfies the desired properties in the theorem.
\end{proof}

\begin{lemma}\label{lemma:lipschitz-approx}
    Suppose $f(x) = \phi(v^\T x)$ is an $L$-Lipschitz function for $v \in \sph$, $\phi: \R \to [-1, 1]$, and $L \geq 1$.
    For any $t \in [1, \sqrt{d} - 1]$ and $\delta \in (0, 1)$, there exists a neural network $g$ of width $O(\frac{tL}{\delta})$ such that:
    \begin{enumerate}
        \item $\rnorm{g} = O(\frac{tL^2}{\delta})$;
        \item $\abs{f(x) - g(x)} \leq \delta$ for all $x$ with $\abs{v^\T x} \leq t$; 
        \item $\abs{f(x) - g(x)} \leq 1$ for all $x \in \R^d$;
        \item $g(x) = 0$ for all $x$ with $\abs{v^\T x} \geq t + \frac1L$; and 
        \item $g$ is a ridge function that in direction $v$.
    \end{enumerate} 
\end{lemma}
\begin{proof}
    We first introduce an $L$-Lipschitz function $\phi_t$ (visualized in Figure~\ref{fig:phi-t}) that perfectly fits $\phi$ on the interval $[-t, t]$ and is zero in $(\infty, -t-\frac1L] \cup [t+\frac1L, \infty]$:
    \[\phi_t(z) := \begin{cases}
        \phi(z) & \text{if $z \in [-t, t]$} ; \\
        \sign{\phi(t)} \max\set{\abs{\phi(t)} - L(z - t)), 0} & \text{if $z \geq t$} ; \\
        \sign{\phi(-t)} \max\set{\abs{\phi(-t)} - L(-z + t)), 0} & \text{if $z \leq -t$} .
    \end{cases}\]
    
    \begin{figure}
        \centering
        \includegraphics[width=0.8\textwidth]{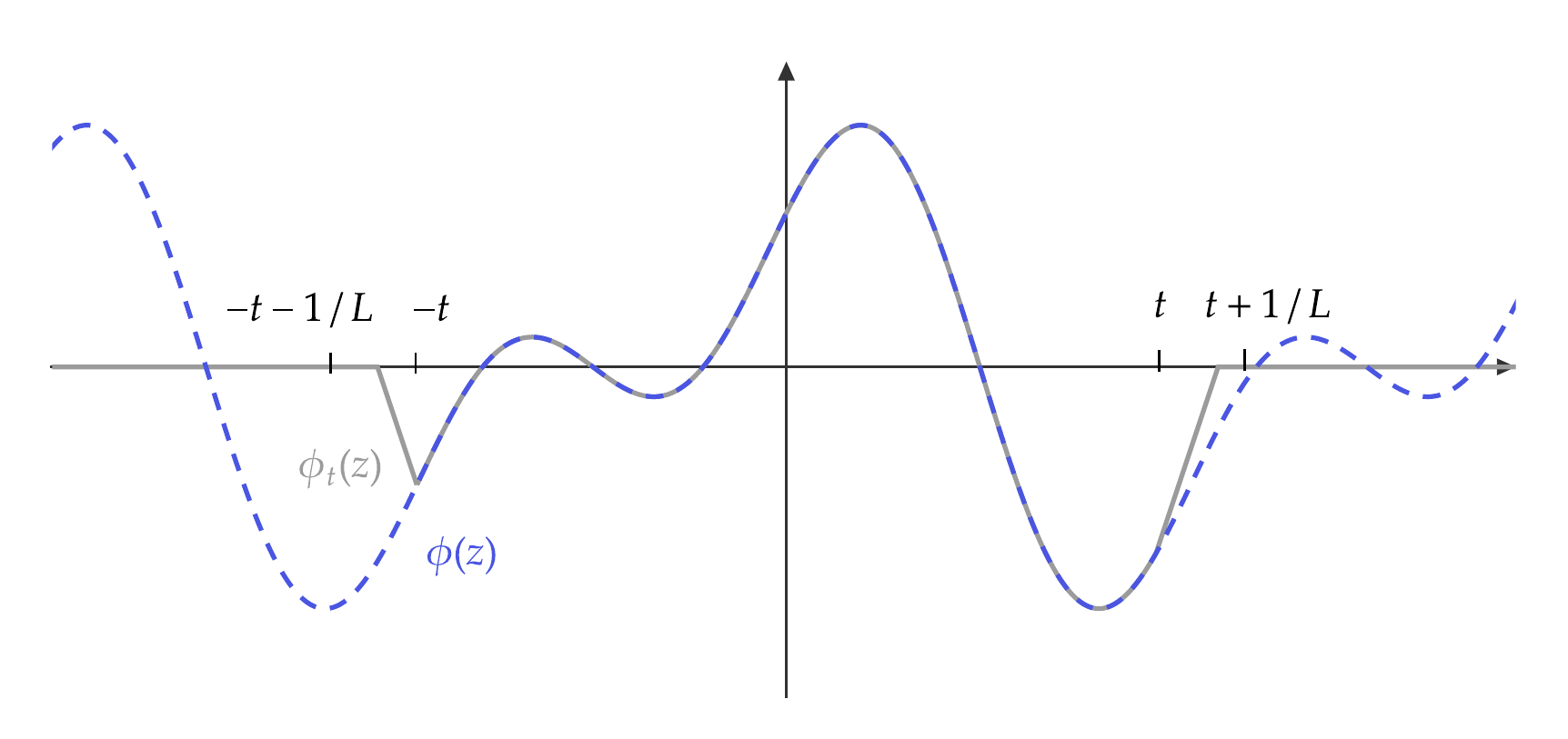}
        \caption{A visualization of how the truncated $\phi_t$ (gray) is generated from $\phi$ (blue), $t$, and $L$.}
        \label{fig:phi-t}
    \end{figure}
    
    Then, there exists a piecewise-linear function $\psi_t$ that 
    \begin{itemize}
        \item point-wise approximates $\phi_t$ to accuracy $\delta$;
        \item has $\psi_t(z) = \phi_t(z)$ for all $z \not\in [-t, t]$;
        \item has $\frac{2tL}{\delta}$ evenly-spaced knots on the interval $[-t, t]$ where $\psi_t$ exactly fits $\phi_t$; and
        \item is $L$-Lipschitz.
    \end{itemize}
    
    As a result $\psi_t$ can be written as a neural network with $\psi_t(z) = \sum_{j=1}^m a^{(j)} \relu{z - b^{(j)}}$ where $m = \frac{2Lt}\delta$, $b^{(j)} \in [-t - \frac1L, t + \frac1L]$, and $\abs{a^{(j)}} \leq 2L$.
    
    By taking $g(x) := \psi_t(v^\T x)$, we have a neural network that satisfies conditions 2, 3, 4, and 5. 
    The bound on $\rnorm{g}$ is immediate from the fact that $g$ can be expressed as a neural network with $O(\frac{tL}{\delta})$ neurons with unit weights, biases in $[-\sqrt{d}, \sqrt{d}]$, and bounded coefficients $a^{(j)}$.
\end{proof}

\subsection{Proof of Theorem~\ref{thm:ridge-cosine-lb}}
\label{assec:periodic-rnorm-ridge-approx-lb}

\thmridgecosinelb* 
\begin{proof}
We prove the claim by a reduction to Theorem~\ref{thm:ridge-parity-lb}.
That is, we show that an interpolant with better \Rnorm than the bound stipulates can be used to construct a neural network that contradicts Theorem~\ref{thm:ridge-parity-lb}.

To do so, we consider a lower dimension $d' = 4\floorl{d/ 4q} - 4$ and create a mapping from points $z \in \flip^{d'}$ to $x_z \in \flip^d$.
We define $a \in [0, 4q - 1]$ such that $2a \equiv d \pmod{4q}$.
For any $z$, we define $x_z$ as follows:
\[x_z = (\underbrace{z_1, \dots, z_1}_{q}, \dots, \underbrace{z_{d'}, \dots, z_{d'}}_{q}, \underbrace{1, \dots, 1}_{a}, \underbrace{-1, \dots, -1}_{d - d'q - a}).\]
Observe that \[\vec1^\T x_z =  q \vec1^\T z + 2a -d + d'q \equiv q \vec1^\T z \pmod{4q}.\] Due to the periodicity of cosine and the fact that  $d'$ is a multiple of 4, \[\cos\parenl{\frac{2\pi}{\rho} v^\T x_z} = \cos\parenl{\frac{\pi}{2} \vec1^\T z} = \parity(z).\]

Consider some $g(x) = \phi(w^\T x)$ with $\norml[\infty]{g - \cos\parenl{\frac{2\pi}{\rho} v^\T \cdot}} \leq \frac12$.
Define $w' \in \R^{d'}$ such that $w'_i := \sum_{j=1}^q w_{(i-1)q + j}$. Observe that $\norm[2]{w'} \leq \sqrt{q}$ and that $w^\T x_z = w^{\prime \T} z + c_w$, where $c_w$ depends only on the remaining elements of $w$.
Define $\tilde{g}(z) = \phi(w^{\prime \T} z + c_w)$.
Then, \[\absl{\tilde{g}(z) - \chi(z)} = \absl{\phi(w^\T x_z) -  \cos\parenl{\frac{2\pi}{\rho} v^\T x_z}} \leq \frac12\] for all $z \in \flip^{d'}$.
Since translation can only decrease the \Rnorm (by exhausting some neurons to effectively behave linearly in the domain) namely, $\rnorm{\tilde{g}} \le \norm[2]{w'} \tv{\phi'} =  \norm[2]{w'} \rnorm{g}$, Theorem~\ref{thm:ridge-parity-lb} implies that $\rnorm{g} = \Omega({d^{\prime 3/2}}/{\sqrt{q}})$.
The theorem statement follows by plugging in $q$ and $d'$.
\end{proof}

\end{document}